\documentclass[12pt]{article} 
\usepackage{arxiv,times}
\usepackage[utf8]{inputenc}
\usepackage[T1]{fontenc}
\usepackage{natbib}
\setcitestyle{authoryear,round,citesep={;},aysep={,},yysep={;}}

\usepackage{amsmath,amsfonts,bm}

\usepackage{graphicx} 
\usepackage{amsmath,amssymb,amsthm,mathtools}
\usepackage{booktabs}
\usepackage{microtype}

\newtheorem{theorem}{Theorem}
\newtheorem{proposition}[theorem]{Proposition}
\newtheorem{lemma}[theorem]{Lemma}

\newcommand{\ip}[2]{\left\langle #1,#2\right\rangle_n}
\newcommand{\norm}[1]{\left\lVert #1\right\rVert}
\newcommand{\ddt}[1]{\frac{\mathrm{d}#1}{\mathrm{d}t}}

\def\eqref#1{equation~(\ref{#1})}
\def\Eqref#1{Equation~(\ref{#1})}

\def\1{\bm{1}}

\DeclareMathAlphabet{\mathsfit}{\encodingdefault}{\sfdefault}{m}{sl}
\SetMathAlphabet{\mathsfit}{bold}{\encodingdefault}{\sfdefault}{bx}{n}

\newcommand{\R}{\mathbb{R}}

\usepackage{hyperref}
\usepackage{url}
\usepackage{amsmath,amssymb,amsthm,mathtools,bm}
\usepackage{graphicx,booktabs,microtype}
\usepackage{xcolor}
\usepackage[nameinlink,capitalize,noabbrev]{cleveref}
\theoremstyle{definition}
\newcommand{\relu}{\operatorname{ReLU}}
\newcommand{\sgn}{\operatorname{sgn}}
\newcommand{\op}{\mathrm{op}}
\newcommand{\qs}{\mathrm{qs}}
\newcommand{\Acal}{\mathcal A}
\newcommand{\VA}{\mathcal V_A}
\newcommand{\epsm}{\varepsilon_m}

\graphicspath{{figures/}}

\title{%
  \fontsize{17pt}{20pt}\selectfont
  Weight Decay and Neuron Condensation: \\
  A Three-Stage Analysis of Two-Layer ReLU Networks
}

\author{
  \parbox{0.95\textwidth}{
    \centering
    \normalfont
    \textbf{Cheng Xu}$^{1,\dagger}$,\quad
    \textbf{Pengxiao Lin}$^{1,\dagger}$,\quad
    \textbf{Zhangchen Zhou}$^{1,\dagger}$,\quad
    \textbf{Zhi-Qin John Xu}$^{1,2,*}$
    \\[0.8em]
    $^{1}$School of Mathematical Sciences,
    Shanghai Jiao Tong University\\[0.3em]
    $^{2}$Institute of Natural Sciences, MOE-LSC,
    Shanghai Jiao Tong University\\[0.8em]
    \small{
    $^{\dagger}$These authors contributed equally to this work.\\[0.2em]
    $^{*}$Corresponding author: \texttt{xuzhiqin@sjtu.edu.cn}}
  }
}

\date{}
\renewcommand{\undertitle}{}
\renewcommand{\shorttitle}{Weight Decay and Neuron Condensation}

\begin{document}

\maketitle

\begin{abstract}
Weight decay is widely used as a regularization technique in neural network training, yet its role in neuron condensation (parameter direction alignment) remains unclear. Starting from a parameter initialization in the neural tangent kernel regime, we characterize training dynamics under weight decay through three stages: rapid fitting, amplitude compression, and neuron condensation. Using a two-layer ReLU network, we analyze a residual correlation field that governs both neuron amplitudes and directions.
During rapid fitting, the residual approaches a quasi-static equilibrium maintained by weight decay while the tangent kernel remains nearly unchanged.
In the early stage of amplitude compression, kernel decay amplifies the residual correlation field, whose isolated attracting extrema provide candidates for condensation directions.
As neuron amplitudes stabilize, we bound the drift of attracting extrema and demonstrate contraction of neuron directions around them, leading to neuron condensation.
This staged analysis provides a dynamical understanding of how
weight decay promotes a condensed representation, beyond reducing parameter norms.
\end{abstract}


\section{Introduction}
\label{sec:intro}

An overparameterized network fits the same training data in many ways,
and which solution gradient-based training finds depends on the
parameterization, the initialization and the optimizer
\citep{luo2021phase,chizat2019lazy,chizat2020implicit}.
Classical capacity control does not explain why such networks
generalize \citep{zhang2017rethinking}; one explanation is that
training returns solutions simpler than the size of the network
suggests \citep{xu2025overview,xu2024overview}.  For a neural network this
simplicity has a concrete form: the learned function may need far fewer
neurons than the network has.

Condensation is how this happens during training: the
input weight vectors of many neurons, biases included, concentrate
around a few directions \citep{luo2021phase,zhou2022condensation,xu2025overview}.
Neurons sharing a direction compute proportional features and act as
one, so a wide network behaves like a much narrower one.

Small initialization induces condensation
\citep{maennel2018quantize,zhou2022condensation}.
The price is speed: the smaller the initialization, the longer training
takes to leave the neighborhood of the origin before fitting begins
\citep{zhang2021embedding,zhang2022embedding}.  Dropout  \citep{srivastava2014dropout} induces condensation as well even with large initialization
\citep{zhang2024dropout}.

A large initialization places a sufficiently wide network in the lazy
regime instead: training fits the data quickly, but the features move
little relative to their initial scale, so no condensation occurs
\citep{jacot2018ntk,chizat2019lazy,luo2021phase}.  We study whether
weight decay induces condensation even with large initialization, in two-layer
ReLU regression.  Sparse minimizers are known to exist for related
norm-regularized formulations \citep{savarese2019infinite,parhi2021banach},
but their existence says nothing about whether the gradient-based learning can lead the network to such solutions.  
In related homogeneous models, training with large initialization and weak weight decay can transition from kernel behavior to a parameter-space maximum-margin bias\citep{lyu2024dichotomy}.  This transition provides one explanation of
grokking, in which test performance improves long after the training
loss has stopped decreasing
\citep{power2022grokking,liu2023omnigrok,kumar2024grokking,lyu2024dichotomy}.
Consistent with this separation of time scales, our experiments
show rapid initial fitting governed by the neural tangent kernel,
followed by slower representation dynamics under weight decay.
During this later evolution, neuron amplitudes shrink and surviving
directions concentrate into a small number of groups.
We characterize this condensation mechanism through three stages:
residual equilibration, amplification of the residual correlation
field, and neuron condensation as amplitudes stabilize. 

The key in this paper is not the proof of theorems but the connection between the weight decay and the weight directions, with reasonable assumptions motivated by the experiments and dynamical analysis. We complement the two-layer analysis with empirical observations in Transformers. Figure~\ref{fig:transformer_condensation} shows condensation in the query and key projection weights of a Transformer model trained on an anchor function task \citep{zhang2024initialization,zhang2026}. Each input contains an ordered pair of anchor tokens whose composition determines the target. All examples generated from the anchor pairs $(3,4)$ and $(4,3)$ are excluded from training and used only for evaluation. These observations extend the empirical evidence beyond two-layer networks.

\begin{figure}[!htbp]
  \centering
  \includegraphics[width=\linewidth]{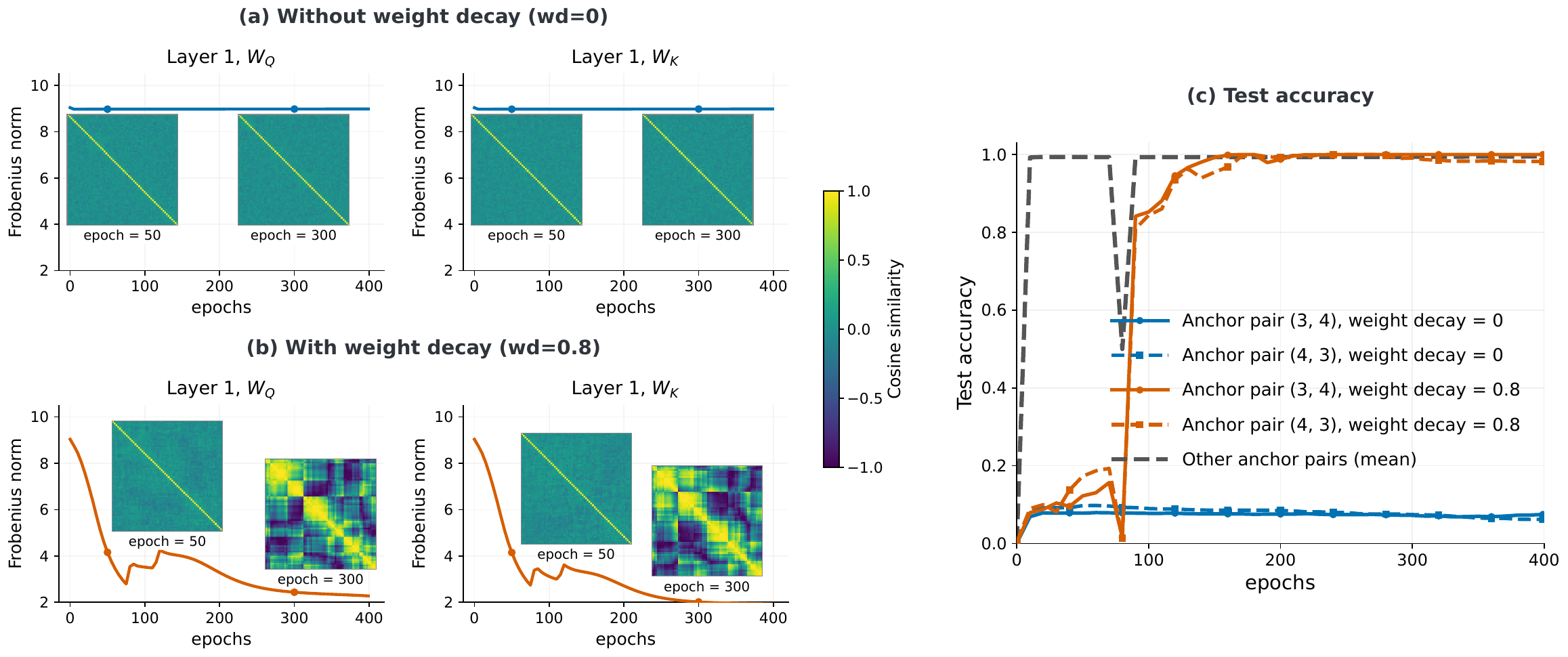}
  \caption{Parameter norm contraction, condensation and generalization in a Transformer.
    Panels (a) and (b) show the first layer $W_Q$ and $W_K$ weights for $\lambda=0$ and $\lambda=0.8$, respectively. Curves show their Frobenius norms, and insets show row-wise cosine similarities at epochs 50 and 300. Weight decay contracts the norms and induces row vectors alignment. Panel (c) shows that $\lambda=0.8$ achieves substantially higher generalization accuracy on the held-out anchor pairs $(3,4)$ and $(4,3)$ than $\lambda=0$. Experimental details are provided in Appendix~\ref{app:transformer}}
  \label{fig:transformer_condensation}
\end{figure}

\paragraph{Main contributions.}
We explain how weight decay induces neuron condensation in two-layer
ReLU networks under large initialization by characterizing three
stages of training: rapid fitting, amplitude compression, and
neuron condensation.
We show that the angular dynamics induced by the field admit only finitely many isolated attracting directions.
During early amplitude compression, weight decay reduces the
neural tangent kernel and amplifies this field.
As amplitudes stabilize, we characterize the slow motion of
attracting extrema and exponential angular contraction around them.
Together, these results provide a dynamical explanation of how
weight decay promotes neuron condensation under large initialization,
distinct from previous analyses of condensation under small
initialization and lazy training under large initialization.

\section{Related Work}
\label{sec:related-work}
\paragraph{Condensation and early directional alignment.}
Small-initialization training can concentrate neurons along a limited
set of directions, as studied through feature quantization
\citep{maennel2018quantize}, initialization-dependent regimes
\citep{luo2021phase}, and initial condensation
\citep{zhou2022condensation,chen2024initial}.
For shallow ReLU networks, \citet{boursier2022gradient} characterize
alignment and subsequent training under orthogonal inputs while \citet{boursier2025alignment} give a quantitative analysis
of early alignment for more general data and investigate its sparsity
bias phenomenon. \citet{kumar2024directional} relate directional convergence near small
initializations and certain saddles to stationary directions of a
neural correlation function.

\paragraph{Weight decay and the transition beyond lazy training.}
The neural tangent kernel and lazy-training analyses describe regimes
in which training is approximated by a linearized model
\citep{jacot2018ntk,chizat2019lazy}.
With $L_2$ regularization, \citet{lewkowycz2020training} derive
time-dependent kernel dynamics for homogeneous networks in an
infinite-width regime and analyze the increase of the
empirical loss.  \citet{lyu2024dichotomy} prove a sharp transition from kernel to a parameter space maximum-margin bias
for large initialization and vanishing weight decay; their general theory assumes smooth homogeneous models.
\citet{kumar2024grokking} also connect delayed generalization to
late feature learning, including examples without regularization. Focusing on convex neural network \citep{bach2017breaking}, 
\citet{pilanci2020convex} establish convex reformulations of
two-layer ReLU training with a block-sparsity penalty and $L^2$ regularization.

\paragraph{Other mechanisms leading to condensation.}
Dropout provides another route to neuron condensation
\citep{zhang2024dropout}. Empirical work on Transformers also reports
condensation under changes in initialization and weight decay
\citep{zhang2024initialization}, while \citet{chen2025rankcollapse}
analyze staged condensation and rank collapse from small
initialization in a different architecture.
Weight decay has additionally been linked to neural collapse in
classification \citep{jacot2025collapse} and low-rank structure at
stationary points \citep{kuzborskij2026lowrank}.
These geometric properties differ from the within-group angular
diameter studied here: concentration of sample representations or
matrix rank alone does not characterize alignment of individual
hidden-neuron directions.

\section{Preliminaries}
\label{sec:setup}

\subsection{Basic notation}
We write $\norm{\cdot}_2$ for the Euclidean norm and
$\norm{\cdot}_{\op}$ for its induced matrix norm. For $\boldsymbol{x},\boldsymbol{y}\in\R^n$,
the empirical inner product and norm are
\[
 \ip {\boldsymbol{x}}{\boldsymbol{y}}=\frac1n \boldsymbol{x}^\top \boldsymbol{y},\qquad
 \norm {\boldsymbol{x}}_n=\sqrt{\ip {\boldsymbol{x}}{\boldsymbol{x}}}=\frac{\norm{\boldsymbol{x}}_2}{\sqrt n}.
\]
The notation $O_p(c_m)$ and $o_p(c_m)$ means that division by $c_m>0$
gives, respectively, a quantity bounded or converging
to zero in probability as $m\to\infty$.
\subsection{Network and training objective}
Given samples $\{(x_i,y_i)\}_{i=1}^n\subset\R\times\R$, consider a
two-layer ReLU network with $m$ hidden neurons:
\begin{equation}
 f_\Theta(x)=\sum_{k=1}^m a_k\relu(w_kx+b_k),
 \qquad \Theta=(a_k,w_k,b_k)_{k=1}^m.
 \label{eq:model}
\end{equation}
For $\lambda>0$, the training objective with weight decay is
\begin{equation}
 \mathcal L(\Theta)=\frac1n\sum_{i=1}^n(f_\Theta(x_i)-y_i)^2,
 \qquad \Phi(\Theta)=\mathcal L(\Theta)
 +\frac\lambda2\sum_{k=1}^m(a_k^2+w_k^2+b_k^2).
 \label{eq:objective}
\end{equation}
Write $\bm{X}=(x_i)_{i=1}^n$, $\boldsymbol{y}=(y_i)_{i=1}^n$, and
${\bm{e}}(t)=f_{\Theta(t)}(\bm{X})-\boldsymbol{y} \in \mathbb{R}^n$ for the training residual.
The parameters $a_k(0),w_k(0),b_k(0)$ are initialized independently from
$N(0,m^{-2\gamma})$, with large initialization $\gamma<1/2$.
We study gradient flow $\ddt{\Theta(t)}=-\nabla_\Theta\Phi(\Theta(t))$,
with weight decay applied to all parameters. Our main experiment uses the target function $f_\star(x)=0.2\relu(x-1/3)+0.2\relu(-x-1/3)$ as shown in Figure~\ref{fig:overview}(b). Detailed experimental settings are provided in Appendix~\ref{app:controls}.

\subsection{Amplitude and angular dynamics}
Write ${\bm{v}}_k=(w_k,b_k)^\top=r_k{\bm{u}}(\theta_k)$, where
$r_k=\norm{{\bm{v}}_k}_2$ and ${\bm{u}}(\theta)=(\cos\theta,\sin\theta)^\top$.
For $r_k>0$, define the direction $\theta_k\in [0,2\pi)$,
the amplitude $A_k=|a_k|r_k$, and $s_k=\sgn(a_k)$.
Positive and negative neurons have $a_k>0$ and $a_k<0$ respectively.
With the directional feature
\begin{equation}
 \phi_\theta(x)=\relu(x\cos\theta+\sin\theta),
 \qquad
 \psi_\theta(x)=\mathbf1_{\{x\cos\theta+\sin\theta>0\}}
                   (-x\sin\theta+\cos\theta),
 \label{eq:directionalfeatures}
\end{equation}
the network is $f_\Theta(x)=\sum_k s_kA_k\phi_{\theta_k}(x)$ by positive homogeneity of $\relu$ activation.
Here $\psi_\theta$ is the angular derivative of $\phi_\theta$ away from activation
boundaries, where $x_i\cos\theta+\sin\theta=0$ for some training input.

For a residual vector ${\bm{e}}\in\R^n$, define the \emph{residual correlation
field} as the scalar function $q_e$ of the angle $\theta$, with
angular derivative $p_e$:
\begin{equation}
 q_e(\theta)=\ip {\bm{e}}{{\bm{\phi}}_\theta({\bm{X}})},\qquad
 p_e(\theta)=\partial_\theta q_e(\theta)=\ip {\bm{e}}{{\bm{\psi}}_\theta({\bm{X}})}.
 \label{eq:qp}
\end{equation} 
During training, write $q_t=q_{e(t)}$,
$p_t=p_{e(t)}$, $q_k=q_t(\theta_k)$, and $p_k=p_t(\theta_k)$.
Let $J_t=\partial_\Theta f_{\Theta(t)}({\bm{X}})$ and $K_t=J_tJ_t^\top/n$
be the Jacobian and empirical neural tangent kernel (NTK).
We have the following gradient flow.

\begin{proposition}[Gradient flow]
\label{prop:flow}
For $r_k>0$, gradient flow for \eqref{eq:objective} satisfies
\begin{equation}
 \ddt{a_k}=-2r_kq_k-\lambda a_k,\qquad
 \ddt{r_k}=-2a_kq_k-\lambda r_k,\qquad
 \ddt{\theta_k}=-2\frac{a_k}{r_k}p_k.
 \label{eq:radial}
\end{equation}
The training residual satisfies
\begin{equation}
 \ddt {\bm{e}}=-2(K_t+\lambda I_n){\bm{e}}-2\lambda {\bm{y}}.
 \label{eq:residual}
\end{equation}
Moreover, the scale imbalance obeys
\begin{equation}
 B_k(t):=a_k(t)^2-r_k(t)^2={\mathrm{e}}^{-2\lambda t}B_k(0).
 \label{eq:balance}
\end{equation}
\end{proposition}

Weight decay acts directly on parameter scales, while angular motion
is governed by $p_k$.
\Eqref{eq:balance} shows that the imbalance $a_k^2-r_k^2$
decays exponentially, motivating a balanced description of late training.
For an initially strictly inactive neuron, $w_kx_i+b_k<0$ for every
training input, so $q_k=p_k=0$: its amplitude decays as
$A_k(t)={\mathrm{e}}^{-2\lambda t}A_k(0)$ while its direction remains fixed.
Proposition ~\ref{prop:flow} is derived in
Appendix~\ref{app:dynamics}.

Holding $K_t$ fixed in \eqref{eq:residual} gives the
\emph{quasi-static residual}
\begin{equation}
 {\bm{e}}_\qs(t):=-\lambda(K_t+\lambda I_n)^{-1}{\bm{y}}.
 \label{eq:equilibrium}
\end{equation}
It is an instantaneous reference for the actual residual ${\bm{e}}(t)$.

\subsection{Neuron condensation}
For a fixed group $G$ of at least two neurons with $A_k(t)>0$, define its angular diameter by
\begin{equation}
 \operatorname{diam}_{\theta}G(t)
 :=\max_{i,j\in G}|\theta_i(t)-\theta_j(t)|.
 \label{eq:angulardiameter}
\end{equation}
We say that $G$ undergoes \emph{neuron condensation} if
$\operatorname{diam}_{\theta}G(t)\to0$ as $t\to\infty$.
Thus condensation means that distinct neurons approach a common
direction. Different groups may condense
around different directions.

\subsection{Three-stage overview}

\begin{figure}[!htbp]
 \centering\includegraphics[width=\linewidth]{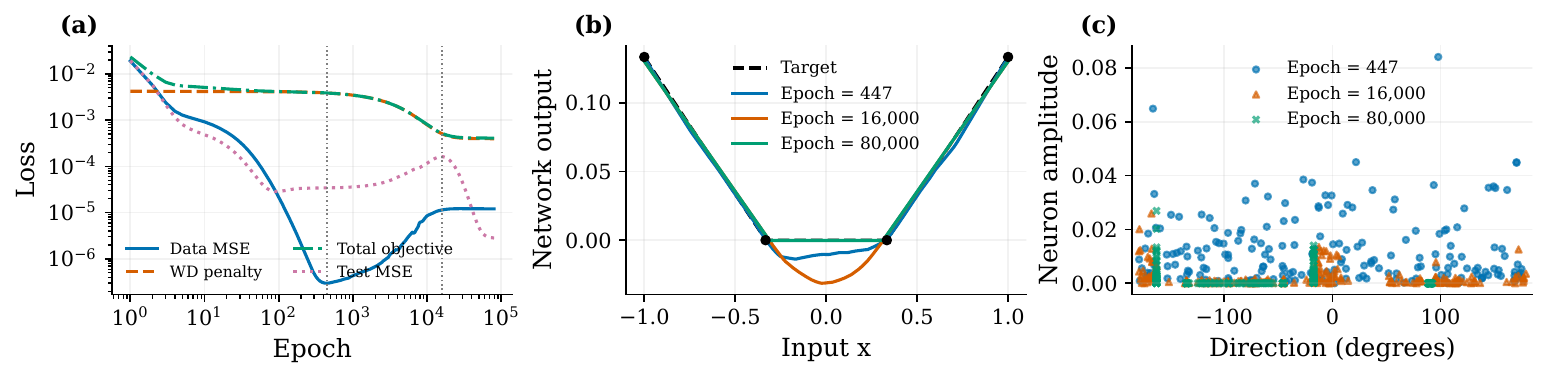}
 \caption{Training overview. (a) Data MSE, weight decay penalty, total
 objective, and test MSE; vertical markers indicate the loss minimum
 (epoch 447) and an epoch near the amplitude plateau (epoch 16,000).
 (b) Learned functions at epochs 447, 16,000, and 80,000. The four black dots are training samples. (c) Directions and amplitudes of all 200 neurons
 at different epochs.}
 \label{fig:overview}
\end{figure}

Figure~\ref{fig:overview} illustrates the evolution from rapid
data fitting to neuron condensation.
The delayed decrease in test loss in panel (a) reveals a
grokking phenomenon.
Panel (b) shows the learned function gradually approaching
the target as the network evolves toward a condensed solution that closely approximates the target function.
The neuron feature snapshots in panel (c) show a complex representation with widely scattered features (blue dots) after rapid fitting. As amplitudes stabilize, the features develop a clearer directional structure (orange dots), and eventually form tight angular groups (green dots) corresponding to the condensed solution.

Figure~\ref{fig:three_stages} resolves this evolution into three stages:
\begin{enumerate}
 \item \textbf{Stage I: Fitting.}
 The data loss decreases rapidly (a), and the residual approaches
 its quasi-static reference (b), while neuron directions remain
 broadly distributed (f).

 \item \textbf{Stage II: Compression.}
 Neuron amplitudes continue to decrease until they stabilize (c), while the attracting minima of the residual correlation field deepen substantially (d),
 strengthening directional attraction. 
 Neuron directions begin to narrow but remain dispersed at the
 end of this stage (f).

 \item \textbf{Stage III: Condensation.}
 After amplitudes have nearly stabilized (c),
 the angular diameters of the two neuron groups continue to decrease
 approximately exponentially (e), with neurons condensing to specific directions (f).
\end{enumerate}

Section~\ref{sec:theory} characterizes the residual left after fitting, the attracting directions and amplification of the residual correlation field, and how amplitude stabilization controls extrema drift and enables angular contraction.

\begin{figure}[!htbp]
 \centering\includegraphics[width=\linewidth]{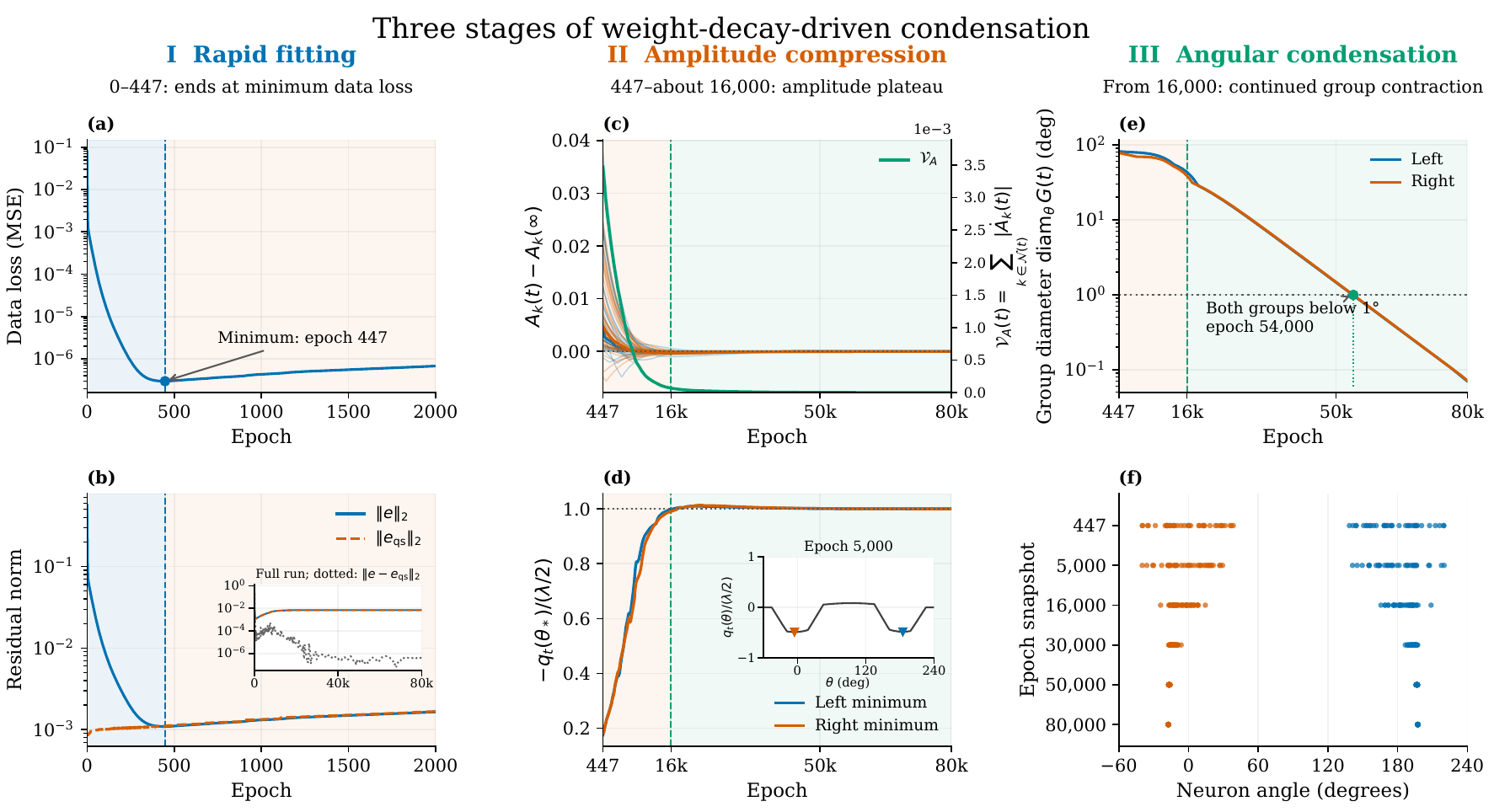}
 \caption{Diagnostics of the three stages.
(a,b) Data MSE and residual norms $\|{\bm{e}}\|_2$, $\|{\bm{e}}_{\qs}\|_2$;
the inset in (b) shows the full run, including
$\|{\bm{e}}-{\bm{e}}_{\qs}\|_2$ (dotted).
(c) Individual amplitude differences $A_k(t)-A_k(\infty)$
(left axis), with $A_k(\infty)$ approximated by the value at
epoch 80,000, and total amplitude variation
$\mathcal V_A(t)=\sum_{k\in\mathcal N(t)}|\dot A_k(t)|$
over all non-dead neurons (green curve, right axis).
(d) Normalized depths $-q_t(\theta_*(t))/(\lambda/2)$ of the
smooth negative minima associated with the two neuron groups.
The inset shows $q_t(\theta)/(\lambda/2)$ at epoch 5,000. The normalized depth plateau in Stage III can be explained by the balanced amplitude dynamics in Appendix~\ref{app:balance}.
(e,f) Angular diameters and angle snapshots of two neuron groups.
}
 \label{fig:three_stages}
\end{figure}

\section{Main results}
\label{sec:theory}
We analyze how the residual correlation field governs neuron dynamics across the three training stages. We first identify the residual left after fast fitting, then
establish the discreteness of the field's attracting extrema
and characterize how amplitude compression amplifies the field. The main condensation theorem
connects small amplitude variation to exponential contraction of a
group's angular diameter.

\subsection{Stage I: Fast residual equilibration}
\label{sec:stage1}
Under the initialization in Section~\ref{sec:setup}, let
$\rho_m=m^{1-2\gamma}$ and $\epsm=\lambda/\rho_m$.
The initial kernel satisfies $K_0/\rho_m\to\bar K\succ0$
in probability as $m\to\infty$; see
Appendix~\ref{app:initialkernel} for the proof. The initial kernel has eigenvalues of order $\rho_m=m^{1-2\gamma}$,
which diverges to infinity since $\gamma<1/2$.
The following theorem shows that the residual rapidly approaches
its quasi-static reference while the kernel remains close
to its initial value.

\begin{theorem}[Residual equilibration]
\label{thm:stage1}
Assume $x_i\neq x_j$ for all $i\neq j$. Fix $\lambda>0$ and $1/4\le\gamma<1/2$. Let
$\kappa_m=\lambda_{\min}(K_0)$ and, on the event $\kappa_m>0$, define
\[
 T_{\mathrm I}=
 \frac{\log(\log m/\epsm)}{2(\kappa_m+\lambda)}.
\]
As $m\to\infty$, the nonlinear gradient flow satisfies
\begin{align}
 T_{\mathrm I}&=O_p(\rho_m^{-1}\log(\rho_m/\lambda))=o_p(1),\\
 \sup_{t\le T_{\mathrm I}}\frac{\norm{K_t-K_0}_{\op}}{\rho_m}
 &=O_p(\rho_m^{-1}(\log m)^{3/2})=o_p(1),\label{eq:kernelstable}\\
 {\bm{e}}(T_{\mathrm I})&=-\epsm\bar K^{-1}{\bm{y}}+o_p(\epsm),\qquad
 \norm{{\bm{e}}(T_{\mathrm I})-{\bm{e}}_\qs(T_{\mathrm I})}_2=o_p(\epsm).
 \label{eq:stage1result}
\end{align}
\end{theorem}
The proof is given in Appendix~\ref{app:stage1}.
Figure~\ref{fig:stage1} shows the evolution of the NTK eigenvalues at finite width. The bound $\gamma\ge1/4$ is imposed to ensure that the initial
residual satisfies $\|{\bm{e}}(0)\|_2=O_p(1)$ in the infinite width limit.
For any finite width network, the initial residual is finite.

\Eqref{eq:kernelstable} shows that the kernel changes by
$o_p(\rho_m)$ during initial fitting, a vanishing fraction of its
initial scale. Throughout this stage, its eigenvalues therefore
remain much larger than the fixed weight decay coefficient $\lambda$.
Thus, the residual relaxes on a time
scale $T_{\mathrm I}=o_p(1)$ is much faster than the effect of weight decay.
For $\lambda>0$ and ${\bm{y}}\ne\bm{0}$, the forcing term $-2\lambda {\bm{y}}$ sustains the quasi-static residual
${\bm{e}}_\qs(t)=-\lambda(K_t+\lambda I_n)^{-1}{\bm{y}}$.
At the end of initial fitting, \eqref{eq:stage1result} shows that
the tracking error $o_p(\epsm)$ can be neglected. Thus the actual residual retains
this nonzero leading term, which generates the correlation
field governing subsequent neuron motion.

\subsection{Stage II: Discrete extrema and amplification of $q_e$}
\label{sec:stage2}

Let $\tilde {\bm{x}}_i=(x_i,1)^\top$ and define the angular boundary set
\[
 \mathcal B=\{\theta \in [0,2\pi):\tilde {\bm{x}}_i^\top {\bm{u}}(\theta)=0
                  \text{ for some }i\}.
\]
An \emph{activation cell} $I$ is a connected component of
$\mathbb{S}^1 \setminus\mathcal B$.
Its active sample set $\mathcal S_I=\{i:\tilde {\bm{x}}_i^\top {\bm{u}}(\theta)>0,\; \forall \theta \in I\}$
is the same for every $\theta\in I$; a cell is dead if this set is empty.
Define the residual moment on each cell by
\[
 {\bm{M}}_I({\bm{e}})=\frac1n\sum_{i\in\mathcal S_I}e_i\tilde {\bm{x}}_i.
\]
Within $I$, the residual correlation field has the sinusoidal form
\begin{equation}
 q_e(\theta)
 =M_{I,1}({\bm{e}})\cos\theta+M_{I,2}({\bm{e}})\sin\theta,
 \label{eq:sinusoid}
\end{equation}
where $M_{I,i}({\bm{e}})$ denotes the $i$th component of ${\bm{M}}_I({\bm{e}})$.

Recall that $\dot\theta_k=-2\beta_k p_e(\theta_k)$, where
$\beta_k=a_k/r_k$ and $p_e=\partial_\theta q_e$.
From \eqref{eq:sinusoid} we have $q_e''(\theta)=-q_e(\theta)$.
For fixed residual ${\bm{e}}$ and ratio $\beta_k$,
the angular dynamics have equilibrium points where
$p_e(\theta_*)=0$.
Within a cell with ${\bm{M}}_I({\bm{e}})\ne \bm{0}$, these points are 
the smooth extrema of $q_e$.
An equilibrium is locally attracting if and only if
$a_kq_e(\theta_*)<0$.
Thus positive neurons are attracted to negative minima,
and negative neurons to positive maxima.
The following proposition characterizes the discrete set of
extrema that can serve as condensation directions.

\begin{proposition}[Discrete extrema]
\label{prop:geometry}
Fix a residual ${\bm{e}}$ and suppose ${\bm{M}}_I({\bm{e}})\ne \bm{0}$ on every
non-dead activation cell. Let $\mathcal E({\bm{e}})$ denote the set
of strict local extrema of $q_e$, and define
\[
 \mathcal E_{\mathrm{int}}({\bm{e}})
 :=\mathcal E({\bm{e}})\setminus\mathcal B,
 \qquad
 \mathcal E_{\mathrm{bd}}({\bm{e}})
 :=\mathcal E({\bm{e}})\cap\mathcal B.
\]
These are, respectively, the smooth extrema inside activation
cells and the extrema at activation boundaries, and satisfy
\[
 1\le|\mathcal E({\bm{e}})|\le4n,
 \qquad
 |\mathcal E_{\mathrm{int}}({\bm{e}})|,\,
 |\mathcal E_{\mathrm{bd}}({\bm{e}})|\le2n.
\]
\end{proposition}

The proof and a discussion of boundary extrema are given in
Appendix~\ref{app:geometry}.
We now focus on smooth extrema and their attraction strength.
Consider a smooth negative minimum $\theta_*(t)$ and a positive
neuron in the same activation cell.
Writing $\delta_k(t)=\theta_k(t)-\theta_*(t)$, the sinusoidal
form of the field \eqref{eq:sinusoid} gives
\[
 \ddt{\delta_k}
 =2\beta_k(t)q_t(\theta_*(t))\sin\delta_k
 -\ddt{\theta_*}.
\]
Since $\beta_k(t)q_t(\theta_*(t))<0$, the first term
drives the neuron toward the minimum.
For a fixed scale ratio, a more negative value of
$q_t(\theta_*(t))$ produces stronger attraction.

We next examine how weight decay changes the field's strength
during amplitude compression.
Theorem~\ref{thm:stage1} provides initial tracking of ${\bm{e}}_\qs$.
The following result shows that slow kernel decay maintains this
tracking and amplifies the residual correlation field during early
amplitude compression. Set
$C_\phi:=\sup_{\theta\notin\mathcal B}\norm{{\bm{\phi}}_\theta({\bm{X}})}_n$.

\begin{theorem}[Kernel decay and field amplification]
\label{thm:earlycompression}
Let ${\bm{y}}\ne\bm{0}$, $t_0=T_{\mathrm I}$ and consider $[t_0,t_1]$ with
$\lambda(t_1-t_0)\le L$, where $L$ is fixed.
Suppose the activation masks remain fixed and every non-dead
neuron satisfies
\[
 B^{-1}\le\left|\frac{a_k(t)}{r_k(t)}\right|\le B,
 \qquad
 |q_t(\theta_k(t))|+|p_t(\theta_k(t))|
 \le\varepsilon\lambda.
\]
Assume $K_{t_0}\succ0$ and
\[
 \frac{\norm{K_{t_0}}_{\op}}{\lambda_{\min}(K_{t_0})}\le\chi,
 \qquad
 \frac{\lambda}{\lambda_{\min}(K_{t_0})}\le\varepsilon,
\]
where $B,\chi$ are fixed and $\varepsilon$ is sufficiently small.
Set $\delta_0=\norm{{\bm{e}}(t_0)-{\bm{e}}_\qs(t_0)}_n/\norm{{\bm{e}}_\qs(t_0)}_n$.
Then, writing $\tau=t-t_0$, there exists $C=C(B,L,\chi)$ such that
\begin{align}
 \norm{K_t-{\mathrm{e}}^{-2\lambda\tau}K_{t_0}}_{\op}
 &\le C\varepsilon {\mathrm{e}}^{-2\lambda\tau}
       \norm{K_{t_0}}_{\op},
 \label{eq:earlykerneldecay}\\
 \sup_{s\in[t_0,t_1]}
 \frac{\norm{{\bm{e}}(s)-{\bm{e}}_\qs(s)}_n}{\norm{{\bm{e}}_\qs(s)}_n}
 &\le C(\delta_0+\varepsilon),
 \label{eq:compressiontracking}\\
 \sup_\theta
 \left|q_t(\theta)-{\mathrm{e}}^{2\lambda\tau}q_{t_0}(\theta)\right|
 &\le C(\varepsilon+\delta_0){\mathrm{e}}^{2\lambda\tau}
       C_\phi\norm{{\bm{e}}_\qs(t_0)}_n.
 \label{eq:earlyfieldstretch}
\end{align}
\end{theorem}

The proof is given in Appendix~\ref{app:earlycompression}. Stage I gives $\delta_0=o_p(1)$ at $T_{\mathrm I}$, and
\eqref{eq:compressiontracking} controls residual tracking
throughout the stated compression interval.
Consistent with this result, Figure~\ref{fig:three_stages}(b)
shows that ${\bm{e}}(t)$ remains close to ${\bm{e}}_{\mathrm{qs}}(t)$
after rapid fitting.
When the correlations are small relative to $\lambda$,
weight decay dominates parameter evolution and approximately
rescales the kernel by ${\mathrm{e}}^{-2\lambda\tau}$. As weight decay reduces the kernel,
${\bm{e}}_\qs(t)\approx-\lambda K_t^{-1}{\bm{y}}$ grows, amplifying the correlation field.

The leading order scaling of weight decay changes the depths of the extrema
but leaves their locations unchanged. However, in the actual dynamics,
Figure~\ref{fig:q} shows that the smooth extrema also move during Stage II.
The following section controls this motion, which results in neuron condensation.

\begin{figure}[!t]
 \centering\includegraphics[width=0.90\linewidth]{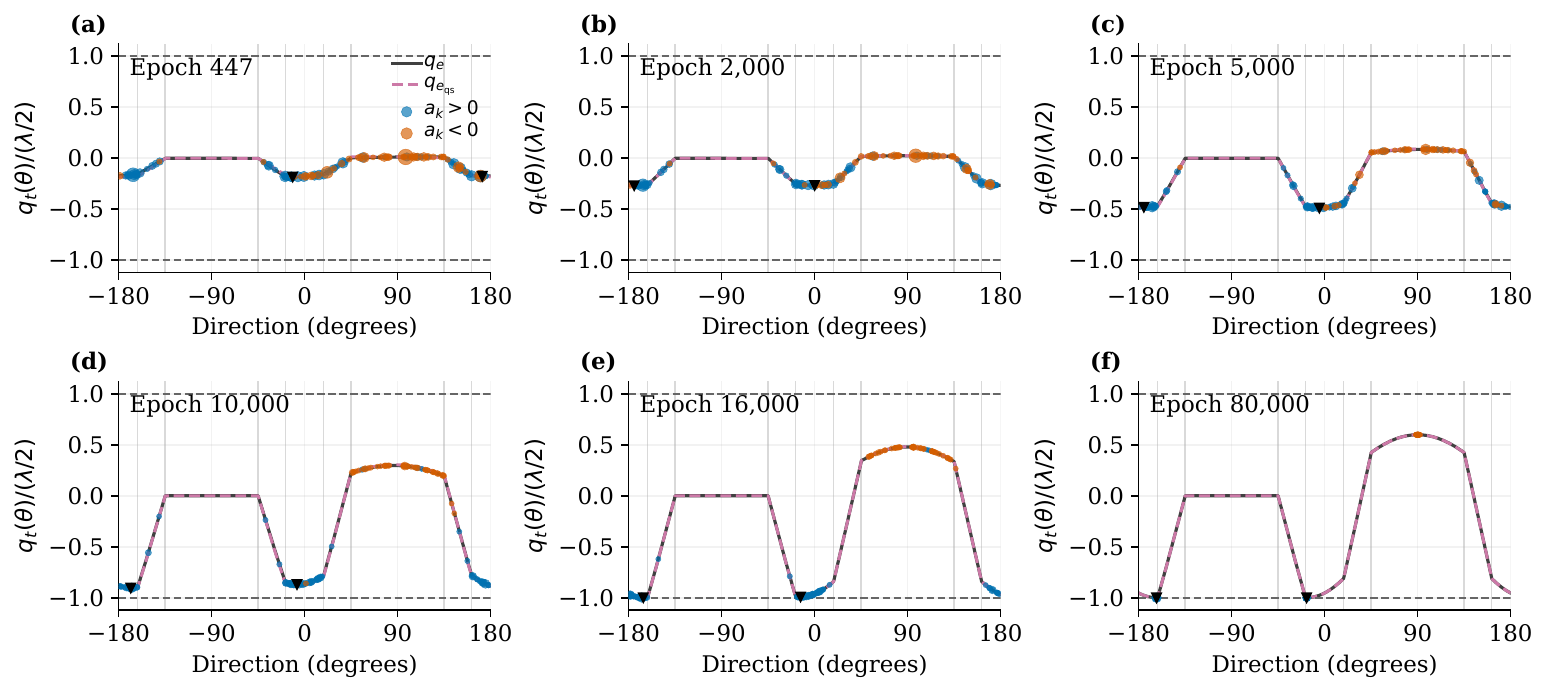}
 \caption{Residual correlation field evolution and surviving neurons.
Curves show normalized $q_t(\theta)/(\lambda/2)$, computed from the actual residual
${\bm{e}}$ (black solid) and its quasi-stationary approximation
${\bm{e}}_{\mathrm{qs}}$ (pink dashed), which nearly overlap.
Blue and orange dots represent positive and negative neurons  respectively. Black triangles mark smooth local minima.
During Stage II (up to epoch 16,000), the field deepens substantially,
while neuron directions remain dispersed around its extrema.
During Stage III, neurons condense at these extrema.}
 \label{fig:q}
\end{figure}

\subsection{Stage III: Slow drift of extrema and neuron condensation}
\label{sec:stage3}

Let $\mathcal N(t)=\{k:{\bm{\phi}}_{\theta_k(t)}({\bm{X}})\ne\bm{0}\}$ be the set of
non-dead neurons. Define their total amplitude, total amplitude
variation rate, and angular energy by
\begin{equation}
 \Acal(t)=\sum_{k\in\mathcal N(t)}A_k(t),\qquad
 \VA(t)=\sum_{k\in\mathcal N(t)}\left|\ddt{A_k(t)}\right|,\qquad
 P(t)=\sum_{k\in\mathcal N(t)}A_k(t)p_k(t)^2.
 \label{eq:energydefs}
\end{equation}
We also use $C_\psi:=\sup_{\theta\notin\mathcal B}
\norm{{\bm{\psi}}_\theta({\bm{X}})}_n$ and $C_\phi=\sup_{\theta\notin\mathcal B}\norm{{\bm{\phi}}_\theta({\bm{X}})}_n$.

Attracting extrema of the residual correlation field evolve
during training.
To understand whether neurons can track these extrema and
condense around them, we first bound their drift.
The following theorem shows how amplitude variation $\VA$
controls the angular energy $P$, and how these quantities
jointly bound the velocity of the residual correlation field extrema.

\begin{theorem}[Field drift]
\label{thm:energy}
Suppose that, for all $t\ge T$, the gradient flow remains in
fixed open activation cells and every non-dead neuron satisfies
$a_k^2=r_k^2=A_k>0$.
Set $\overline{\Acal}_T:=\Phi(\Theta(T))/\lambda$.
Then, for all $t\ge T$,
\begin{align}
 P(t)
 &\le {\mathrm{e}}^{-2\lambda(t-T)}P(T)
 +\frac{C_\phi^2}{4}
   \int_T^t {\mathrm{e}}^{-2\lambda(t-s)}\VA(s)^2\,ds,
 \label{eq:Pbound}\\
 \norm{\ddt{{\bm{e}}(t)}}_n
 &\le C_\phi\VA(t)
 +2C_\psi\sqrt{\overline{\Acal}_T P(t)}.
 \label{eq:debound}
\end{align}
If $\theta_*(t)$ is a differentiable branch of smooth negative
minima in a fixed cell with depth $R(t) = -q_t(\theta_*(t))\ge R_0>0$ for all $t\ge T$,
then
\begin{equation}
 \left|\ddt{\theta_*(t)}\right|
 \le \frac{C_\psi}{R_0}
 \left[
 C_\phi\VA(t)
 +2C_\psi\sqrt{\overline{\Acal}_T P(t)}
 \right].
 \label{eq:starspeed}
\end{equation}
\end{theorem}

The proof is given in Appendix~\ref{app:driftproof}.
Theorem~\ref{thm:energy} bounds the drift of attracting minima
by the amplitude variation $\VA(t)$ and the angular energy $P(t)$.
The latter is itself controlled by amplitude variation:
if $\VA(t)\le\nu_A$ for all $t\ge T$, then \eqref{eq:Pbound} gives
\begin{equation}
 P(t)\le {\mathrm{e}}^{-2\lambda(t-T)}P(T)
       +\frac{C_\phi^2\nu_A^2}{8\lambda}.
 \label{eq:angualr energe}
\end{equation}
The first term in \eqref{eq:angualr energe} captures the influence of the initial angular energy and decays exponentially.
Once it falls below fixed $\varepsilon^2$,
substituting the resulting bound on $P(t)$ and
$\VA(t)\le\nu_A$ into \eqref{eq:starspeed} yields
\begin{equation}
\begin{aligned}
 \left|\ddt{\theta_*(t)}\right|
 &\le \frac{C_\psi}{R_0}
 \left[
 C_\phi\nu_A
 +2C_\psi\sqrt{\overline{\Acal}_T
 \left(\varepsilon^2+\frac{C_\phi^2\nu_A^2}{8\lambda}\right)}
 \right] \\
 &\le \frac{C_\psi C_\phi}{R_0}
 \left(1+C_\psi\sqrt{\frac{\overline{\Acal}_T}{2\lambda}}\right)\nu_A
 +\frac{2C_\psi^2\sqrt{\overline{\Acal}_T}}{R_0}\varepsilon.
\end{aligned}
 \label{eq:starspeed-nuA}
\end{equation}
Thus, the drift speed is bounded by a term proportional to $\nu_A$.
Figure~\ref{fig:three_stages}(c) shows that $\VA(t)$ remains
small during Stage III, supporting a small upper bound $\nu_A$
over this period and motivating the slow extrema drift.

Weight decay also reduces the absolute imbalance $a_k^2-r_k^2$
through \eqref{eq:balance}, motivating the balanced model for
surviving neurons.
The next theorem shows that sufficiently slow drift keeps nearby
neurons in an attracting band and yields angular contraction.

\begin{theorem}[Angular contraction]
\label{thm:condensation}
Let $G$ be a fixed group of at least two positive neurons and
$\theta_*(t)$ a differentiable branch of smooth negative minima.
For all $t\ge T$, assume that $a_k(t)=r_k(t)>0$ for every $k\in G$,
the minimum depth satisfies $R(t)\ge R_0>0$, and the band
$[\theta_*(t)-d,\theta_*(t)+d]$ lies inside one fixed open
activation cell, where $0<d<\pi/2$.
If
\[
 |\theta_k(T)-\theta_*(T)|\le d\quad(k\in G),
 \qquad
 \sup_{t\ge T}\left|\ddt{\theta_*(t)}\right|<2R_0\sin d,
\]
then every neuron in $G$ remains in the band and, with
$c_d=2R_0\cos d>0$,
\begin{equation}
 \operatorname{diam}_{\theta}G(t)
 \le {\mathrm{e}}^{-c_d(t-T)}\operatorname{diam}_{\theta}G(T),
 \qquad t\ge T.
 \label{eq:condensation}
\end{equation}
Consequently, $G$ undergoes neuron condensation.
If additionally $\ddt{\theta_*(t)}\to0$, then
$|\theta_k(t)-\theta_*(t)|\to0$ for every $k\in G$.
\end{theorem}

The proof is given in Appendix~\ref{app:condensation}.
The theorem shows that a group can condense around a moving
minimum: sufficiently slow drift retains the neurons in the
attracting band, where their angular diameter decreases
exponentially. 

Together with Theorem~\ref{thm:energy}, this gives the connection
between amplitude stabilization and condensation.
Small amplitude variation and a decayed initial angular-energy
contribution yield a small drift bound.
Once this bound is below $2R_0\sin d$, the contraction theorem
applies to neurons then inside the band, as long as the stated
conditions persist. This describes how angular concentration
can continue during Stage III after amplitudes vary little.
The analogous result holds for exactly balanced negative
neurons near a smooth positive maximum.

\section{Discussion}
\label{sec:discussion}

Experiments with additional target functions in
Appendix~\ref{app:multiknot} show the same progression
from rapid fitting to late condensation.
Our analysis characterizes neuron condensation around attracting
extrema of the residual correlation field, but does not establish
that these directions coincide with those in the target representation.
The additional experiments also exhibit more condensation directions
than those in the target representation.
Appendix~\ref{app:zero-decay} compares large initialization with and without weight decay against small initialization
without weight decay.
Both large initialization with weight decay and small initialization
produce condensed solutions, but the former exhibits a distinct
rapid fitting phase before condensation.
Our dynamical analysis focuses on two-layer ReLU networks, while a corresponding dynamical theory for Transformers remains open.

\section{Conclusion}
\label{sec:conclusion}

Our work explains a mechanism for neuron condensation in two-layer
ReLU networks trained with weight decay under large initialization.
This mechanism differs from persistent lazy training at large
initialization and condensation at small initialization. Our experiments show that
rapid fitting precedes substantial directional alignment: the network
first fits the data, then develops a condensed representation during
continued training with weight decay.

By analyzing amplitude and angular dynamics, we characterize three
stages: rapid fitting, amplitude compression, and neuron condensation. 
During rapid fitting, the residual approaches a quasi-static equilibrium maintained by weight decay. 
In the early stage of amplitude compression, we show that the residual correlation field admits only finitely many
isolated attracting extrema and that early amplitude compression amplifies the field.
As neuron amplitudes stabilize, we prove that neuron directions condense around slowly moving attracting directions.
This analysis links weight decay to the formation of a condensed representation, clarifying its role beyond parameter norm reduction.

\newpage

\bibliographystyle{reference}
\bibliography{reference}
\clearpage
\appendix

\section{Neuron dynamics}
\label{app:dynamics}

\subsection{Neuron dynamics and balance}
\label{app:flow}
We derive Proposition~\ref{prop:flow} using the loss normalization
$\mathcal L=\norm {\bm{e}}_n^2$ in \eqref{eq:objective}.
Write ${\bm{u}}_k={\bm{u}}(\theta_k)$,
${\bm{u}}_{\perp,k}={\bm{u}}_\perp(\theta_k)=(-\sin\theta_k,\cos\theta_k)^\top$,
let $I_k$ be the activation cell containing $\theta_k$, and set
$\chi_{ki}=\mathbf1_{\{{\bm{v}}_k^\top\tilde {\bm{x}}_i>0\}}$ inside an open cell. Define
\[
 {\bm{\phi}}_k(t):={\bm{\phi}}_{\theta_k(t)}({\bm{X}}),\qquad
 {\bm{\psi}}_k(t):={\bm{\psi}}_{\theta_k(t)}({\bm{X}}).
\]
The derivatives of a training prediction are
\[
 \partial_{a_k}f_\Theta(x_i)=r_k({\bm{\phi}}_k)_i,\qquad
 \nabla_{{\bm{v}}_k}f_\Theta(x_i)=a_k\chi_{ki}\tilde {\bm{x}}_i.
\]
Consequently, with ${\bm{M}}_k=n^{-1}\sum_i e_i\chi_{ki}\tilde {\bm{x}}_i
={\bm{M}}_{I_k}({\bm{e}})$,
\begin{align}
 \partial_{a_k}\mathcal L
 &=\frac2n\sum_i e_i r_k({\bm{\phi}}_k)_i=2r_kq_k,\notag\\
 \nabla_{{\bm{v}}_k}\mathcal L&=\frac2n\sum_i e_i a_k\chi_{ki}\tilde {\bm{x}}_i
 =2a_k{\bm{M}}_k,\notag\\
 \ddt{a_{k}}&=-2r_kq_k-\lambda a_k,\qquad
 \ddt{{\bm{v}}_{k}}=-2a_k{\bm{M}}_k-\lambda {\bm{v}}_k.
 \label{eq:cartesianflow}
\end{align}
In particular, the two components of the last equation are
\[
 \ddt{w_{k}}=-\frac{2a_k}{n}\sum_i e_i\chi_{ki}x_i-\lambda w_k,
 \qquad
 \ddt{b_{k}}=-\frac{2a_k}{n}\sum_i e_i\chi_{ki}-\lambda b_k.
\]

\paragraph{Radial and angular projections.}
For $r_k>0$, differentiating ${\bm{v}}_k=r_k{\bm{u}}_k$ gives
$\ddt{{\bm{v}}_{k}}=\ddt{r_{k}}{\bm{u}}_k+r_k\ddt{\theta_{k}}{\bm{u}}_{\perp,k}$.
The orthonormality of ${\bm{u}}_k,{\bm{u}}_{\perp,k}$ and the identities
${\bm{u}}_k^\top {\bm{M}}_k=q_k$, ${\bm{u}}_{\perp,k}^\top {\bm{M}}_k=p_k$ imply
\begin{align}
 \ddt{r_{k}}&={\bm{u}}_k^\top\ddt{{\bm{v}}_{k}}=-2a_kq_k-\lambda r_k,\notag\\
 r_k\ddt{\theta_{k}}&={\bm{u}}_{\perp,k}^\top\ddt{{\bm{v}}_{k}}=-2a_kp_k.
 \label{eq:polarprojection}
\end{align}
Dividing the second equation by $r_k$ gives the angular equation in
\eqref{eq:radial}. The decay term has zero tangential projection
because ${\bm{u}}_{\perp,k}^\top {\bm{v}}_k=0$.

\paragraph{Balance and inactive neurons.}
The mixed data-gradient terms cancel exactly:
\begin{align}
 \ddt{B_{k}}
 &=2a_k(-2r_kq_k-\lambda a_k)
   -2r_k(-2a_kq_k-\lambda r_k)\notag\\
 &=-2\lambda(a_k^2-r_k^2)=-2\lambda B_k.
\end{align}
Multiplication by ${\mathrm{e}}^{2\lambda t}$ proves \eqref{eq:balance}.
It also gives the absolute scale bound
\[
 \bigl||a_k(t)|-r_k(t)\bigr|^2
 \le |a_k(t)^2-r_k(t)^2|
 ={\mathrm{e}}^{-2\lambda t}|B_k(0)|.
\]
If ${\bm{v}}_k(0)^\top\tilde {\bm{x}}_i<0$ for every $i$, all data gradients vanish.
The solution is $a_k(t)={\mathrm{e}}^{-\lambda t}a_k(0)$ and
${\bm{v}}_k(t)={\mathrm{e}}^{-\lambda t}{\bm{v}}_k(0)$. Every strict inactive inequality
therefore persists, $\theta_k$ stays fixed when defined, and
$A_k(t)={\mathrm{e}}^{-2\lambda t}A_k(0)$.

\paragraph{Residual dynamics.}
The parameter equation is $\ddt{\Theta}=-2J_t^\top {\bm{e}}/n-\lambda\Theta$.
For each sample, the contribution of neuron $k$ to $J_t\Theta$ is
\[
 a_k\partial_{a_k}f_\Theta(x_i)
 +{\bm{v}}_k^\top\nabla_{{\bm{v}}_k}f_\Theta(x_i)
 =2a_k\relu({\bm{v}}_k^\top\tilde {\bm{x}}_i).
\]
Thus $J_t\Theta=2f_\Theta({\bm{X}})$, and the chain rule gives
\begin{align}
 \ddt{{\bm{e}}}=J_t\ddt{\Theta}
 &=-\frac2nJ_tJ_t^\top {\bm{e}}-\lambda J_t\Theta\notag\\
 &=-2K_t{\bm{e}}-2\lambda f_\Theta({\bm{X}})
 =-2(K_t+\lambda I_n){\bm{e}}-2\lambda {\bm{y}}.
\end{align}
For fixed $K\succeq0$, setting the residual velocity to zero gives
${\bm{e}}_{\mathrm{eq}}(K)=-\lambda(K+\lambda I_n)^{-1}{\bm{y}}$.

At zero preactivation, use the same selection $\chi_{ki}\in[0,1]$
in the parameter flow and Jacobian, with $p_k$ denoting the selected
angular gradient. Along an absolutely continuous trajectory,
preactivation derivatives vanish almost everywhere on their zero sets;
the ReLU chain rule and the identities above therefore hold almost
everywhere, including boundary contacts. In particular,
\[
 \frac{\mathrm d}{\mathrm dt}\Phi(\Theta(t))
 =-\norm{\ddt{\Theta}}_2^2\le0.
\]

\subsection{General and balanced amplitude dynamics}
\label{app:balance}
For $r_k>0$, the quotient rule and \eqref{eq:radial} give
\begin{align}
 \ddt{\beta_{k}}
 &=\frac{\ddt{a_{k}}r_k-a_k\ddt{r_{k}}}{r_k^2}\notag\\
 &=\frac{(-2r_kq_k-\lambda a_k)r_k
       -a_k(-2a_kq_k-\lambda r_k)}{r_k^2}
 =2q_k(\beta_k^2-1).
\end{align}
On any interval where $a_k\ne0$, its sign $s_k$ is constant and
$A_k=|a_k|r_k>0$. The logarithmic derivative is
\begin{align}
 \frac{\ddt{A_{k}}}{A_k}
 &=\frac{\ddt{a_{k}}}{a_k}+\frac{\ddt{r_{k}}}{r_k}\notag\\
 &=-2\frac{r_k}{a_k}q_k-\lambda
   -2\frac{a_k}{r_k}q_k-\lambda
 =-2\lambda-2q_k(\beta_k+\beta_k^{-1}).
\end{align}
Equivalently, $\ddt{A_{k}}=-2\lambda A_k-2s_k(a_k^2+r_k^2)q_k$.
If $B_k(T)=0$, the balance identity preserves $|a_k|=r_k$;
for a nonzero neuron this means $\beta_k=s_k$ and $A_k=a_k^2=r_k^2$.
Substitution gives the balanced dynamics
\begin{equation}
 \ddt{A_k}=-2A_k(\lambda+2s_kq_k),\qquad
 \ddt{\theta_k}=-2s_kp_k.
 \label{eq:balanced}
\end{equation}
In this regime,
\[
 A_k(t)=A_k(T)\exp\!\left[-2\lambda(t-T)
                    -4s_k\int_T^t q_k(s)\,ds\right],
\]
so a positive amplitude remains positive on every finite interval with
bounded correlations.
The general equations retain the evolving ratio $\beta_k$;
under exact balance, $\beta_k=s_k$ gives the reduced system
\eqref{eq:balanced}. For surviving positive neurons with nearly stationary amplitudes and directions close to attracting minima, these dynamics give
$-2q_t(\theta_*)/\lambda\approx1$, as shown in Figure~\ref{fig:three_stages}(d).

\section{Fast residual equilibration}
\label{app:stage1}

\subsection{Proof of Theorem~\ref{thm:stage1}: initialization}
\label{app:initialkernel}
The proof has three steps: initialization estimates, short-time kernel
stability, and transfer to the residual scale.

Writing $\alpha_{ij}$ for the angle between $\tilde {\bm{x}}_i$ and
$\tilde {\bm{x}}_j$, the limiting kernel is
\begin{equation}
 \bar K_{ij}=
 \frac{\norm{\tilde {\bm{x}}_i}_2\norm{\tilde {\bm{x}}_j}_2}{2\pi n}
 \big[\sin\alpha_{ij}+2(\pi-\alpha_{ij})\cos\alpha_{ij}\big].
 \label{eq:limitkernel}
\end{equation}

Let $\sigma_m=m^{-\gamma}$, and write $a_k(0)=\sigma_m\xi_k$,
${\bm{v}}_k(0)=\sigma_m {\bm{g}}_k$ with independent standard Gaussian $\xi_k$ and
${\bm{g}}_k\in\R^2$. Each output summand has mean zero and variance
\[
 \mathbb E\big[a_k(0)^2\relu({\bm{v}}_k(0)^\top\tilde {\bm{x}}_i)^2\big]
 =\sigma_m^4\mathbb E\xi_k^2\,
   \mathbb E\relu({\bm{g}}_k^\top\tilde {\bm{x}}_i)^2
 =\frac{\sigma_m^4}{2}\norm{\tilde {\bm{x}}_i}_2^2.
\]
Independence across neurons gives
$\operatorname{Var}(f_{\Theta(0)}(x_i))=(1+x_i^2)m^{1-4\gamma}/2$.
For fixed $n,{\bm{y}}$ and $\gamma\ge1/4$, summing these second moments gives
$\norm{{\bm{e}}(0)}_2=O_p(1)$.

With $\chi_{ki}(t)=\mathbf1_{\{{\bm{v}}_k(t)^\top\tilde {\bm{x}}_i>0\}}$, the
empirical kernel entries are
\[
 (K_t)_{ij}=\frac1n\sum_k\left[
 \relu({\bm{v}}_k^\top\tilde {\bm{x}}_i)\relu({\bm{v}}_k^\top\tilde {\bm{x}}_j)
 +a_k^2\chi_{ki}\chi_{kj}\tilde {\bm{x}}_i^\top\tilde {\bm{x}}_j\right].
\]
To compute the expectation, rotate coordinates so that the directions
of $\tilde {\bm{x}}_i,\tilde {\bm{x}}_j$ have angles $0,\alpha$, where
$\alpha=\alpha_{ij}\in[0,\pi]$. Write a standard Gaussian as
${\bm{g}}=R_g(\cos\varphi,\sin\varphi)$, with $\varphi$ uniform on $[0,2\pi)$,
independent of $R_g$, and $\mathbb E R_g^2=2$. The common positive
half-planes overlap on $[\alpha-\pi/2,\pi/2]$, modulo $2\pi$. Hence
\begin{align}
 \mathbb E[\relu({\bm{g}}^\top\tilde {\bm{x}}_i)\relu({\bm{g}}^\top\tilde {\bm{x}}_j)]
 &=\frac{\norm{\tilde {\bm{x}}_i}_2\norm{\tilde {\bm{x}}_j}_2}{\pi}
   \int_{\alpha-\pi/2}^{\pi/2}\cos\varphi\cos(\varphi-\alpha)\,d\varphi
   \notag\\
 &=\frac{\norm{\tilde {\bm{x}}_i}_2\norm{\tilde {\bm{x}}_j}_2}{2\pi}
   [\sin\alpha+(\pi-\alpha)\cos\alpha],\notag\\
 \mathbb P({\bm{g}}^\top\tilde {\bm{x}}_i>0,\ {\bm{g}}^\top\tilde {\bm{x}}_j>0)
 &=\frac{\pi-\alpha}{2\pi}.
\end{align}
The integral follows from
$2\cos\varphi\cos(\varphi-\alpha)=\cos\alpha+\cos(2\varphi-\alpha)$.
Adding the second kernel term, using $\mathbb E\xi_k^2=1$ and
$\tilde {\bm{x}}_i^\top\tilde {\bm{x}}_j=\norm{\tilde {\bm{x}}_i}_2\norm{\tilde {\bm{x}}_j}_2\cos\alpha$,
proves \eqref{eq:limitkernel}.

Since $\rho_m=m\sigma_m^2$, the law of large numbers for the
independent kernel summands gives $K_0/\rho_m\to\bar K$ in
operator norm in probability, with $n$ fixed.
For positivity, a zero quadratic form of the random-feature contribution
would imply
\[
 \sum_i z_i\relu({\bm{g}}^\top\tilde {\bm{x}}_i)=0
\]
for almost every Gaussian ${\bm{g}}$, hence for every ${\bm{g}}$ by continuity.
At a nonzero point on ${\bm{g}}^\top\tilde {\bm{x}}_j=0$, the other terms are smooth
because the inputs are distinct. The gradient jump is therefore
$z_j\tilde {\bm{x}}_j=0$, so every $z_j=0$ and $\bar K\succ0$.

\subsection{Proof of Theorem~\ref{thm:stage1}: kernel stability}
\paragraph{Residual and parameter bounds.}
Let $h_m$ be a deterministic horizon with $h_m\sqrt{\log m}\to0$.
Positivity of $K_t$ gives
\[
 \frac12\frac{\mathrm d}{\mathrm dt}\norm {\bm{e}}_2^2
 =-2{\bm{e}}^\top K_t{\bm{e}}-2\lambda\norm {\bm{e}}_2^2-2\lambda {\bm{e}}^\top {\bm{y}}
 \le2\lambda\norm {\bm{e}}_2\norm {\bm{y}}_2.
\]
Thus $\mathrm d\norm {\bm{e}}_2/\mathrm dt\le2\lambda\norm {\bm{y}}_2$
(in the upper Dini sense at ${\bm{e}}=0$). Since $\norm{{\bm{e}}(0)}_2=O_p(1)$,
the residual is uniformly bounded in probability on $[0,h_m]$.
Localize to an arbitrarily high-probability event on which this bound
and the Gaussian initial-norm bounds below have fixed constants.
With $C_x=(n^{-1}\sum_i\norm{\tilde {\bm{x}}_i}_2^2)^{1/2}$,
\[
 \norm{{\bm{M}}_k}_2\le C_x\norm {\bm{e}}_n,
 \qquad |q_k|\le C_x\norm {\bm{e}}_n.
\]
For ${\bm{\eta}}_k=(a_k,{\bm{v}}_k)$, \eqref{eq:cartesianflow} consequently gives
$\norm{\ddt{{\bm{\eta}}_k}}_2\le C\norm{{\bm{\eta}}_k}_2$, with $C$ independent
of $k,m$. Gronwall's inequality yields
\begin{align*}
 \norm{{\bm{\eta}}_k(t)}_2&\le {\mathrm{e}}^{Ct}\norm{{\bm{\eta}}_k(0)}_2,\\
 \norm{{\bm{\eta}}_k(t)-{\bm{\eta}}_k(0)}_2
 &\le({\mathrm{e}}^{Ct}-1)\norm{{\bm{\eta}}_k(0)}_2.
\end{align*}
In particular,
\[
 \sup_{t\le h_m}\norm{{\bm{\eta}}_k(t)-{\bm{\eta}}_k(0)}_2
 \le Ch_m\sqrt{a_k(0)^2+r_k(0)^2}.
\]
The Gaussian initialization also gives
\[
 \max_k\norm{{\bm{\eta}}_k(0)}_2=O_p(m^{-\gamma}\sqrt{\log m}),
 \qquad
 \frac1m\sum_k m^{2\gamma}\norm{{\bm{\eta}}_k(0)}_2^2=O_p(1).
\]

\paragraph{Boundary contacts and kernel variation.}
A contact with the boundary of sample $i$ can occur only if
\[
 |{\bm{v}}_k(0)^\top\tilde {\bm{x}}_i|\le C m^{-\gamma}h_m\sqrt{\log m}
\]
on the localized event. Lipschitz continuity controls the feature
products and changes in $a_k^2$; changes in activation masks are
bounded by the initial squared output weights on these contact events.
Hence
\begin{align*}
 \sup_{t\le h_m}\frac{\norm{K_t-K_0}_{\op}}{\rho_m}
 &\le \frac{Ch_m}{m}\sum_k m^{2\gamma}(a_k(0)^2+r_k(0)^2)\\
 &\quad+\frac C m\sum_k m^{2\gamma}a_k(0)^2
 \sum_i\mathbf1_{\{|{\bm{v}}_k(0)^\top\tilde {\bm{x}}_i|
 \le C m^{-\gamma}h_m\sqrt{\log m}\}}.
\end{align*}
Independence of $a_k(0)$ and ${\bm{v}}_k(0)$ and the bounded Gaussian
density of ${\bm{g}}_k^\top\tilde {\bm{x}}_i$ imply
\begin{align*}
 \mathbb E\!\left[\xi_k^2
 \mathbf1_{\{|{\bm{g}}_k^\top\tilde {\bm{x}}_i|\le C h_m\sqrt{\log m}\}}\right]
 &=\mathbb E\xi_k^2\,
   \mathbb P(|{\bm{g}}_k^\top\tilde {\bm{x}}_i|\le C h_m\sqrt{\log m})\\
 &\le C' h_m\sqrt{\log m}.
\end{align*}
Since $n$ is fixed, Markov's inequality gives
\[
 \sup_{t\le h_m}\frac{\norm{K_t-K_0}_{\op}}{\rho_m}
 =O_p(h_m\sqrt{\log m}).
\]
Choosing
\[
 h_m=\frac{\rho_m^{-1}}{\lambda_{\min}(\bar K)}
       \log(\log m/\epsm)
\]
ensures $T_{\mathrm I}\le h_m$ with probability tending to one and
proves \eqref{eq:kernelstable}.

\subsection{Proof of Theorem~\ref{thm:stage1}: residual equilibration}
Recall ${\bm{e}}_\qs(0)=-\lambda(K_0+\lambda I_n)^{-1}{\bm{y}}$, and set
$\Delta_m=\sup_{t\le T_{\mathrm I}}\norm{K_t-K_0}_{\op}$.
Then $\Delta_m/\kappa_m=o_p(1)$ and
$\Delta_mT_{\mathrm I}=O_p(\rho_m^{-1}(\log m)^{5/2})=o_p(1)$.
For ${\bm{z}}_0(t)={\bm{e}}(t)-{\bm{e}}_\qs(0)$, subtract the frozen equilibrium relation
to obtain
\[
 \ddt{{\bm{z}}_{0}}=-2(K_t+\lambda I_n){\bm{z}}_0
           -2(K_t-K_0){\bm{e}}_\qs(0).
\]
On $\Delta_m<\kappa_m$, the homogeneous evolution operator $U(t,s)$
satisfies
$\norm{U(t,s)}_{\op}\le {\mathrm{e}}^{-2(\kappa_m+\lambda-\Delta_m)(t-s)}$.
This follows by differentiating the squared norm of a homogeneous
solution and using $K_t\succeq(\kappa_m-\Delta_m)I_n$.
Variation of constants gives
\[
 {\bm{z}}_0(t)=U(t,0){\bm{z}}_0(0)-2\int_0^tU(t,s)(K_s-K_0){\bm{e}}_\qs(0)\,ds.
\]
Bounding the integral by its exponential envelope yields
\begin{align}
 \norm{{\bm{z}}_0(T_{\mathrm I})}_2
 &\le {\mathrm{e}}^{-2(\kappa_m+\lambda-\Delta_m)T_{\mathrm I}}\norm{{\bm{z}}_0(0)}_2
 +\frac{\Delta_m}{\kappa_m+\lambda-\Delta_m}\norm{{\bm{e}}_\qs(0)}_2.
\end{align}
The first term is $O_p(\epsm/\log m){\mathrm{e}}^{2\Delta_mT_{\mathrm I}}$
and the second is $o_p(\epsm)$. Since
${\bm{e}}_\qs(0)=-\epsm\bar K^{-1}{\bm{y}}+o_p(\epsm)$, the first expansion follows.
The resolvent identity gives
\begin{align}
 {\bm{e}}_\qs(t)-{\bm{e}}_\qs(0)
 &=\lambda(K_t+\lambda I_n)^{-1}(K_t-K_0)
                  (K_0+\lambda I_n)^{-1}{\bm{y}},\notag\\
 \norm{{\bm{e}}_\qs(t)-{\bm{e}}_\qs(0)}_2
 &\le\frac{\lambda\Delta_m\norm {\bm{y}}_2}
 { (\kappa_m+\lambda-\Delta_m)(\kappa_m+\lambda)}
 =o_p(\epsm),\qquad t\le T_{\mathrm I}.
\end{align}
Here $\kappa_m/\rho_m\to\lambda_{\min}(\bar K)>0$ and
$\Delta_m/\rho_m=o_p(1)$. Combining this bound with the expansion of
${\bm{e}}(T_{\mathrm I})$ proves its claimed tracking error relative to
${\bm{e}}_\qs(T_{\mathrm I})$ and completes Theorem~\ref{thm:stage1}.

\clearpage
\section{Discrete extrema and amplification of $q_e$}
\label{app:geometry}

\subsection{Proof of Proposition~\ref{prop:geometry}}
\begin{proof}
Each equation $\tilde {\bm{x}}_i^\top {\bm{u}}(\theta)=0$ contributes two
antipodal boundary angles. Thus there are at most $2n$
boundaries and at most $2n$ open activation cells, each of
length at most $\pi$.

On a non-dead cell $I$, the assumption ${\bm{M}}_I({\bm{e}})\ne\bm{0}$ makes
$q_e(\theta)={\bm{M}}_I({\bm{e}})^\top {\bm{u}}(\theta)$ a nonconstant sinusoid.
Its critical points are spaced by $\pi$, so each open cell
contains at most one smooth extremum. A dead cell has
$q_e=0$ and no strict interior extremum. Consequently,
\[
 |\mathcal E_{\mathrm{int}}({\bm{e}})|\le2n,
 \qquad
 |\mathcal E_{\mathrm{bd}}({\bm{e}})|\le|\mathcal B|\le2n.
\]
Every strict local extremum is either inside a cell or at an
activation boundary, giving $|\mathcal E({\bm{e}})|\le4n$.

Finally, $q_e$ is continuous on the circle and nonzero on
some non-dead cell, so it attains a nonzero global maximum
or minimum. The points attaining this value cannot lie in
dead-cell interiors; they belong to the finite set of
interior critical points and activation boundaries.
Each such point is therefore isolated among the global
optimizers and is a strict local extremum.
Hence $|\mathcal E({\bm{e}})|\ge1$.
\end{proof}

Activation boundaries are fixed by the inputs. A boundary extremum
therefore has fixed position while it persists; for positive neurons,
the one-sided condition $p_t(\theta_b^-)<0<p_t(\theta_b^+)$ gives
attraction toward the boundary.

\subsection{Proof of Theorem~\ref{thm:earlycompression}}
\label{app:earlycompression}
\begin{proof}
Write $K_*=K_{t_0}$, $\kappa_*:=\lambda_{\min}(K_*)>0$,
$\tau=t-t_0$, and $\alpha(t)={\mathrm{e}}^{-2\lambda\tau}$.
For the kernel calculation, abbreviate
$q_k=q_t(\theta_k(t))$ and $p_k=p_t(\theta_k(t))$.

\paragraph{Kernel derivative.}
For each neuron, let $\chi_{ki}$ be its fixed activation indicator
on input $i$, and let
\[
 (S_k)_{ij}:=\chi_{ki}\chi_{kj}\tilde {\bm{x}}_i^\top\tilde {\bm{x}}_j.
\]
The polar identity
$\chi_{ki}\tilde {\bm{x}}_i=({\bm{\phi}}_k)_i{\bm{u}}_k+({\bm{\psi}}_k)_i{\bm{u}}_{\perp,k}$ gives
$S_k={\bm{\phi}}_k{\bm{\phi}}_k^\top+{\bm{\psi}}_k{\bm{\psi}}_k^\top$.
The matrix $S_k$ is constant while the mask is fixed. The contribution
of neuron $k$ to the kernel is therefore
\[
 K_k=\frac1n(r_k^2{\bm{\phi}}_k{\bm{\phi}}_k^\top+a_k^2S_k),\qquad
 K_t=\sum_k K_k.
\]
Using $\ddt{{\bm{\phi}}_k}={\bm{\psi}}_k\ddt{\theta_k}$ and the parameter equations,
we obtain
\begin{align}
 \ddt{K_k}&=-2\lambda K_k+E_k,\notag\\
 E_k&=-\frac{4a_kr_kq_k}{n}
                ({\bm{\phi}}_k{\bm{\phi}}_k^\top+S_k)
      -\frac{2a_kr_kp_k}{n}
                ({\bm{\psi}}_k{\bm{\phi}}_k^\top+{\bm{\phi}}_k{\bm{\psi}}_k^\top).
 \label{eq:appkernelderivative}
\end{align}
For any real vectors ${\bm{v}},{\bm{w}}$,
$-({\bm{v}}{\bm{v}}^\top+{\bm{w}}{\bm{w}}^\top)\preceq {\bm{v}}{\bm{w}}^\top+{\bm{w}}{\bm{v}}^\top
\preceq {\bm{v}}{\bm{v}}^\top+{\bm{w}}{\bm{w}}^\top$.
Also $|a_kr_k|\le Br_k^2$ and $|a_kr_k|\le Ba_k^2$ by the ratio
assumption. Thus the first term of $E_k$ is bounded above and below
by $\pm4B|q_k|K_k$, and the second by $\pm2B|p_k|K_k$.
Dead neurons contribute zero to $K_t$. Summing and using
$|q_k|+|p_k|\le\varepsilon\lambda$ gives
\begin{equation}
 \ddt{K_t}=-2\lambda K_t+E(t),\qquad
 -4B\varepsilon\lambda K_t\preceq E(t)
 \preceq4B\varepsilon\lambda K_t.
 \label{eq:appkernelcomparison}
\end{equation}

\paragraph{Kernel decay.}
Let $F(t)=\alpha(t)^{-1}K_t$. For each fixed vector ${\bm{v}}$,
\eqref{eq:appkernelcomparison} implies
\[
 -4B\varepsilon\lambda\,{\bm{v}}^\top F{\bm{v}}
 \le\frac{\mathrm d}{\mathrm dt}({\bm{v}}^\top F{\bm{v}})
 \le4B\varepsilon\lambda\,{\bm{v}}^\top F{\bm{v}}.
\]
Integrating these scalar inequalities yields
\begin{equation}
 {\mathrm{e}}^{-h(t)}K_*\preceq F(t)\preceq {\mathrm{e}}^{h(t)}K_*,\qquad
 h(t):=4B\varepsilon\lambda\tau\le4BL\varepsilon.
 \label{eq:appscaledkernel}
\end{equation}
Consequently,
\[
 \norm{F(t)-K_*}_{\op}
 \le({\mathrm{e}}^{h(t)}-1)\norm{K_*}_{\op}
 \le C_1\varepsilon\norm{K_*}_{\op},
\]
where $C_1$ depends only on $B,L$ for sufficiently small $\varepsilon$.
Multiplication by $\alpha(t)$ proves \eqref{eq:earlykerneldecay}.
In particular, the kernel remains positive definite on the interval.

\paragraph{Quasi-static scaling.}
Using the actual kernel just estimated, set
${\bm{e}}_\qs(t)=-\lambda(K_t+\lambda I_n)^{-1}{\bm{y}}$.
Let $A_0=K_*+\lambda I_n$ and
\[
 D(t)=F(t)-K_*+\lambda(\alpha(t)^{-1}-1)I_n.
\]
Then $K_t+\lambda I_n=\alpha(t)(A_0+D(t))$.
The inverse-matrix identity gives the exact comparison
\begin{equation}
 {\bm{e}}_\qs(t)-\alpha(t)^{-1}{\bm{e}}_\qs(t_0)
 =-\alpha(t)^{-1}(A_0+D(t))^{-1}D(t){\bm{e}}_\qs(t_0).
 \label{eq:appqscomparison}
\end{equation}
By \eqref{eq:appscaledkernel},
$\lambda_{\min}(A_0+D(t))\ge {\mathrm{e}}^{-h(t)}\kappa_*$.
Moreover, using the spectral assumptions and $\alpha^{-1}\le {\mathrm{e}}^{2L}$,
\[
 \frac{\norm{D(t)}_{\op}}{\kappa_*}
 \le \chi({\mathrm{e}}^{h(t)}-1)
       +\varepsilon({\mathrm{e}}^{2L}-1)
 \le C_2\varepsilon,
\]
where $C_2$ depends only on $B,L,\chi$.
Taking empirical norms in \eqref{eq:appqscomparison} therefore gives
\[
 \norm{{\bm{e}}_\qs(t)-\alpha(t)^{-1}{\bm{e}}_\qs(t_0)}_n
 \le C\varepsilon\alpha(t)^{-1}\norm{{\bm{e}}_\qs(t_0)}_n.
\]
For sufficiently small $\varepsilon$, this also gives the uniform comparison
\begin{equation}
 \tfrac12\alpha(t)^{-1}\norm{{\bm{e}}_\qs(t_0)}_n
 \le\norm{{\bm{e}}_\qs(t)}_n
 \le2\alpha(t)^{-1}\norm{{\bm{e}}_\qs(t_0)}_n.
 \label{eq:appqscomparability}
\end{equation}

\paragraph{Residual tracking.}
Set $\underline\kappa={\mathrm{e}}^{-2L-4BL\varepsilon}\kappa_*$.
By \eqref{eq:appscaledkernel}, throughout the interval,
\[
 \lambda_{\min}(K_t)\ge\underline\kappa,\qquad
 \frac{\norm{K_t}_{\op}}{\lambda_{\min}(K_t)}
 \le {\mathrm{e}}^{8BL\varepsilon}\chi,\qquad
 \frac{\lambda}{\underline\kappa}\le C\varepsilon.
\]
From~\eqref{eq:appkernelcomparison},
$\norm{\dot K_t}_{\op}\le(2+4B\varepsilon)\lambda\norm{K_t}_{\op}$.
Differentiating $(K_t+\lambda I_n){\bm{e}}_\qs(t)=-\lambda {\bm{y}}$ therefore yields
\[
 \dot {\bm{e}}_\qs=-(K_t+\lambda I_n)^{-1}\dot K_t {\bm{e}}_\qs,\qquad
 \norm{\dot {\bm{e}}_\qs}_n
 \le\frac{\norm{\dot K_t}_{\op}}{\lambda_{\min}(K_t)+\lambda}
       \norm{{\bm{e}}_\qs}_n
 \le C\lambda\norm{{\bm{e}}_\qs}_n.
\]
Let ${\bm{z}}={\bm{e}}-{\bm{e}}_\qs$. Subtracting the moving equilibrium from
\eqref{eq:residual} gives
\begin{equation}
 \dot {\bm{z}}=-2(K_t+\lambda I_n){\bm{z}}-\dot {\bm{e}}_\qs.
 \label{eq:trackingerror}
\end{equation}
The homogeneous evolution operator satisfies
$\norm{U(t,s)}_{\op}\le {\mathrm{e}}^{-2(\underline\kappa+\lambda)(t-s)}$.
Variation of constants therefore yields
\begin{equation}
 \norm{{\bm{z}}(t)}_n
 \le {\mathrm{e}}^{-2(\underline\kappa+\lambda)\tau}\norm{{\bm{z}}(t_0)}_n
 +\int_{t_0}^t {\mathrm{e}}^{-2(\underline\kappa+\lambda)(t-s)}
                \norm{\dot {\bm{e}}_\qs(s)}_n\,ds.
 \label{eq:zbound}
\end{equation}
Together with \eqref{eq:appqscomparability}, this gives
\begin{align*}
 \norm{{\bm{z}}(t)}_n
 &\le {\mathrm{e}}^{-2(\underline\kappa+\lambda)\tau}\norm{{\bm{z}}(t_0)}_n
 +C\lambda\norm{{\bm{e}}_\qs(t_0)}_n
    \int_{t_0}^t {\mathrm{e}}^{-2(\underline\kappa+\lambda)(t-s)}
                  {\mathrm{e}}^{2\lambda(s-t_0)}\,ds\\
 &\le {\mathrm{e}}^{-2(\underline\kappa+\lambda)\tau}\norm{{\bm{z}}(t_0)}_n
 +\frac{C\lambda}{\underline\kappa+2\lambda}
       \alpha(t)^{-1}\norm{{\bm{e}}_\qs(t_0)}_n.
\end{align*}
Dividing by the lower bound in \eqref{eq:appqscomparability} gives
\begin{equation}
 \frac{\norm{{\bm{z}}(t)}_n}{\norm{{\bm{e}}_\qs(t)}_n}
 \le2{\mathrm{e}}^{-2(\underline\kappa+2\lambda)\tau}\delta_0
      +\frac{C\lambda}{\underline\kappa+2\lambda}
 \le C(\delta_0+\varepsilon),
 \label{eq:appcompressiontracking}
\end{equation}
which proves \eqref{eq:compressiontracking} uniformly on $[t_0,t_1]$.
Thus the initial mismatch relaxes on the kernel time scale, while
motion of the reference on the decay time scale produces only a small lag.

\paragraph{Amplification of the actual field.}
Using the exact decomposition
\[
 {\bm{e}}(t)-\alpha(t)^{-1}{\bm{e}}(t_0)
 ={\bm{e}}_\qs(t)-\alpha(t)^{-1}{\bm{e}}_\qs(t_0)
       +{\bm{z}}(t)-\alpha(t)^{-1}{\bm{z}}(t_0),
\]
the quasi-static scaling estimate and the preceding bound on ${\bm{z}}$ give
\[
 \norm{{\bm{e}}(t)-\alpha(t)^{-1}{\bm{e}}(t_0)}_n
 \le C(\varepsilon+\delta_0)\alpha(t)^{-1}
                    \norm{{\bm{e}}_\qs(t_0)}_n.
\]
Cauchy--Schwarz and the definition of $C_\phi$ now prove
\eqref{eq:earlyfieldstretch} for $q_t=q_{e(t)}$.
The term $\delta_0$ accounts for the initial tracking error;
Theorem~\ref{thm:stage1} gives $\delta_0=o_p(1)$ for ${\bm{y}}\ne\bm{0}$.
\end{proof}

\clearpage
\section{Slow drift of extrema and neuron condensation}
\label{app:energy}

\subsection{Amplitude budget and angular energy}
\label{app:budget}
Under the hypotheses of Theorem~\ref{thm:energy}, $\mathcal N(t)$
is constant for $t\ge T$; write $\mathcal N$ for this set.
Exact balance and energy dissipation give
\begin{equation}
 \lambda\Acal(t)
 =\frac{\lambda}{2}\sum_{k\in\mathcal N}(a_k(t)^2+r_k(t)^2)
 \le\Phi(\Theta(t))\le\Phi(\Theta(T)),\qquad
 t\ge T.
 \label{eq:amplitudebudget}
\end{equation}

Hence $\Acal(t)\le\overline{\Acal}_T$.

\subsection{Proof of Theorem~\ref{thm:energy}}
\label{app:driftproof}
\begin{proof}

Fix any finite $t>T$. On $[T,t]$, write $\mathcal N=\mathcal N(t)$,
which is constant. Under the theorem's balance assumption,
\[
 \ddt{{\bm{e}}}={\bm{v}}_A+{\bm{v}}_\theta,\qquad
 {\bm{v}}_A=\sum_{k\in\mathcal N}s_k\ddt{A_{k}}{\bm{\phi}}_k,\qquad
 {\bm{v}}_\theta=-2\sum_{k\in\mathcal N}A_kp_k{\bm{\psi}}_k.
\]
Here ${\bm{v}}_A,{\bm{v}}_\theta\in\R^n$ are output-velocity vectors, whereas
$\VA$ is the scalar sum of absolute amplitude derivatives.
To see the first identity, differentiate the entire output sum:
\[
 \ddt{{\bm{e}}}=\sum_{k\in\mathcal N}s_k
       (\ddt{A_{k}}{\bm{\phi}}_k+A_k\ddt{\theta_{k}}{\bm{\psi}}_k).
\]
Strictly dead neurons have zero output throughout this interval.
Substituting $\ddt{\theta_{k}}=-2s_kp_k$ cancels the two factors of $s_k$
and gives ${\bm{v}}_\theta$. Within a fixed cell,
$\partial_\theta{\bm{\psi}}_\theta({\bm{X}})=-{\bm{\phi}}_\theta({\bm{X}})$, so the full chain rule is
\begin{align}
 \ddt{p_{k}}
 &=\ip{\ddt{{\bm{e}}}}{{\bm{\psi}}_k}
   +\ip {\bm{e}}{\partial_\theta{\bm{\psi}}_{\theta_k}({\bm{X}})}\ddt{\theta_{k}}\notag\\
 &=\ip{\ddt{{\bm{e}}}}{{\bm{\psi}}_k}-q_k\ddt{\theta_{k}}
 =\ip{\ddt{{\bm{e}}}}{{\bm{\psi}}_k}+2s_kq_kp_k.
\end{align}
In the derivative of $P$, the amplitude term
$-4s_kA_kq_kp_k^2$ therefore cancels the angular chain-rule term
$+4s_kA_kq_kp_k^2$. Summing gives
\begin{align}
 \ddt{P}&=\sum_{k\in\mathcal N}(\ddt{A_{k}}p_k^2+2A_kp_k\ddt{p_{k}})\notag\\
 &=-2\lambda P+2\ip{\ddt{{\bm{e}}}}{\sum_{k\in\mathcal N} A_kp_k{\bm{\psi}}_k}\notag\\
 &=-2\lambda P-\norm{{\bm{v}}_\theta}_n^2-\ip{{\bm{v}}_A}{{\bm{v}}_\theta}\notag\\
 &=-2\lambda P-\norm{{\bm{v}}_\theta+{\bm{v}}_A/2}_n^2+\tfrac14\norm{{\bm{v}}_A}_n^2
 \le-2\lambda P+\tfrac14 C_\phi^2\VA^2.
\end{align}
Here the triangle inequality gives
$\norm{{\bm{v}}_A}_n\le\sum_k\left|\ddt{A_{k}}\right|\norm{{\bm{\phi}}_k}_n\le C_\phi\VA$.
Multiplying the differential inequality by ${\mathrm{e}}^{2\lambda(t-T)}$ and
integrating yields
\[
 P(t)\le {\mathrm{e}}^{-2\lambda(t-T)}P(T)
 +\frac{C_\phi^2}{4}\int_T^t {\mathrm{e}}^{-2\lambda(t-s)}\VA(s)^2\,ds,
\]
which is \eqref{eq:Pbound}. Weighted Cauchy--Schwarz gives the remaining
velocity estimate explicitly:
\begin{align}
 \norm{{\bm{v}}_\theta}_n
 &\le2\sum_k A_k|p_k|\norm{{\bm{\psi}}_k}_n
 \le2C_\psi\sum_k\sqrt{A_k}\sqrt{A_kp_k^2}\notag\\
 &\le2C_\psi\left(\sum_k A_k\right)^{1/2}
               \left(\sum_k A_kp_k^2\right)^{1/2}
 =2C_\psi\sqrt{\Acal P}.
\end{align}
Combining $\norm{\ddt{{\bm{e}}}}_n\le\norm{{\bm{v}}_A}_n+\norm{{\bm{v}}_\theta}_n$ with
\eqref{eq:amplitudebudget} proves \eqref{eq:debound}.

At a smooth minimum, $q_t''(\theta_*)=R>0$ makes its branch locally
differentiable by the implicit function theorem. Differentiating
$p_t(\theta_*(t))=0$ yields
\[
 0=\ip{\ddt{{\bm{e}}}}{{\bm{\psi}}_{\theta_*}({\bm{X}})}+q_t''(\theta_*)\ddt{\theta_{*}}.
\]
Solving for its velocity and applying Cauchy--Schwarz gives
\[
 \ddt{\theta_{*}}=-\frac{\ip{\ddt{{\bm{e}}}}{{\bm{\psi}}_{\theta_*}({\bm{X}})}}{R},\qquad
 \left|\ddt{\theta_{*}}\right|\le\frac{C_\psi}{R_0}\norm{\ddt{{\bm{e}}}}_n,
\]
Combining this bound with \eqref{eq:debound} proves
\eqref{eq:starspeed}.
\end{proof}

\subsection{Consequences of small amplitude variation}
If $\VA(t)\le\nu_A$ for all $t\ge T$, integrating the exponential kernel
in \eqref{eq:Pbound} gives
\begin{equation}
 P(t)\le {\mathrm{e}}^{-2\lambda(t-T)}P(T)
 +\frac{C_\phi^2\nu_A^2}{8\lambda}
       (1-{\mathrm{e}}^{-2\lambda(t-T)})=:H_{\nu_A}(t)^2.
 \label{eq:finiteenergy}
\end{equation}
Combining this with \eqref{eq:debound} and \eqref{eq:starspeed} gives
\begin{equation}
 \left|\ddt{\theta_*(t)}\right|
 \le\frac{C_\psi}{R_0}
 \left[C_\phi\nu_A+2C_\psi\sqrt{\overline{\Acal}_T}
 H_{\nu_A}(t)\right].
 \label{eq:energyfinite}
\end{equation}
Under the hypotheses of Theorem~\ref{thm:energy}, if
$\VA(t)\to0$, its energy bound tends to zero. Indeed, for any
$\eta>0$, choose $S\ge T$ so that $\VA(s)^2\le\eta$ for $s\ge S$.
For $t\ge S$ the convolution is bounded by
\[
 {\mathrm{e}}^{-2\lambda(t-S)}\int_T^S {\mathrm{e}}^{-2\lambda(S-s)}\VA(s)^2\,ds
 +\frac{\eta}{2\lambda}.
\]
Letting $t\to\infty$ and then $\eta\downarrow0$ proves $P(t)\to0$.
The residual and minimum velocities then tend to zero by
\eqref{eq:debound} and \eqref{eq:starspeed}.
A uniform bound $\VA\le\nu_A$ alone instead gives
$\limsup_{t\to\infty}P(t)\le C_\phi^2\nu_A^2/(8\lambda)$.

\subsection{Angular tracking}
\label{app:individual}
\begin{lemma}[Angular tracking]
\label{lem:individual}
Consider an exactly balanced positive-output neuron $a_k=r_k>0$
for all $t\ge T$. Let $\theta_*(t)$ be a differentiable smooth
negative-minimum branch with depth $R(t)\ge R_0>0$ for all $t\ge T$.
Suppose the closed
band $[\theta_*(t)-d,\theta_*(t)+d]$, with $0<d<\pi/2$, lies inside
one fixed open activation cell for all $t\ge T$. If
\[
 |\delta_k(T)|\le d,\qquad
 \sup_{t\ge T}\left|\ddt{\theta_{*}(t)}\right|\le\omega<2R_0\sin d,
 \qquad \delta_k(t)=\theta_k(t)-\theta_*(t),
\]
then the neuron remains in this band. With $c_d=2R_0\cos d>0$,
\begin{align}
 |\delta_k(t)|
 &\le {\mathrm{e}}^{-c_d(t-T)}|\delta_k(T)|
 +\int_T^t {\mathrm{e}}^{-c_d(t-s)}\left|\frac{\mathrm d \theta_{*}(s)}{\mathrm ds}\right|\,ds\label{eq:individualtracking}\\
 &\le {\mathrm{e}}^{-c_d(t-T)}|\delta_k(T)|
 +\frac{\omega}{c_d}\bigl(1-{\mathrm{e}}^{-c_d(t-T)}\bigr).
 \label{eq:individualtube}
\end{align}
\end{lemma}

\begin{proof}
The cell representation $q_t(\theta)=-R(t)\cos(\theta-\theta_*(t))$
gives $p_t(\theta)=R(t)\sin(\theta-\theta_*(t))$. For an exactly
balanced positive-output neuron, differentiating $\delta_k=\theta_k-\theta_*$
therefore gives
\[
 \ddt{\delta_{k}}=-2R(t)\sin\delta_k-\ddt{\theta_{*}}.
\]
At $\delta_k=d$ the relative velocity is at most
$-2R_0\sin d+\omega<0$; at $\delta_k=-d$ it is at least
$2R_0\sin d-\omega>0$. A first-exit argument therefore makes the band
forward invariant for all $t\ge T$. For $|\delta|\le d$,
$\operatorname{sgn}(\delta)\sin\delta\ge\cos d\,|\delta|$.
The upper right Dini derivative therefore satisfies
\[
 D^+|\delta_k|\le-c_d|\delta_k|+\left|\ddt{\theta_{*}}\right|,
\]
including at $\delta_k=0$. Multiplication by
${\mathrm{e}}^{c_d(t-T)}$ and integration proves
\eqref{eq:individualtracking}--\eqref{eq:individualtube}.
\end{proof}

\subsection{Pairwise angular contraction}
\label{app:pairwise}
Recall the angular diameter in \eqref{eq:angulardiameter}.
The following estimate proves \eqref{eq:condensation} once
Lemma~\ref{lem:individual} supplies the invariant band.

\begin{lemma}[Pairwise contraction]
\label{lem:pairwise}
Let $G$ be a fixed group of exactly balanced positive-output neurons,
$a_k=r_k>0$ for all $t\ge T$. Suppose that, for all $t\ge T$,
a smooth negative-minimum branch $\theta_*(t)$ has depth
$R(t)\ge R_0>0$ and all group directions remain in a common band
\[
 |\theta_k(t)-\theta_*(t)|\le d<\pi/2
\]
contained in one open activation cell. Use a common continuous angular
coordinate on this band. Then for $i,j\in G$ and $t\ge T$,
\begin{equation}
 |\theta_i(t)-\theta_j(t)|
 \le {\mathrm{e}}^{-2R_0\cos(d)(t-T)}|\theta_i(T)-\theta_j(T)|.
 \label{eq:pairwisecontract}
\end{equation}
Consequently, $\operatorname{diam}_{\theta}G(t)$ obeys the same contraction
bound and tends to zero, regardless of whether the well converges
to a fixed angle.
\end{lemma}

Taking the difference of two angular equations gives
\begin{equation}
 \frac{d}{dt}(\theta_i-\theta_j)
 =-2R\bigl(\sin\delta_i-\sin\delta_j\bigr).
 \label{eq:pairwisedynamics}
\end{equation}
\begin{proof}

Set $\Delta_{ij}=\theta_i-\theta_j=\delta_i-\delta_j$.
The divided difference
\[
 c_{ij}(t)=\int_0^1
 \cos\bigl(\delta_j(t)+s(\delta_i(t)-\delta_j(t))\bigr)\,ds
\]
is well defined even when the two angles coincide. Since both offsets
lie in $[-d,d]$, it satisfies $c_{ij}\ge\cos d>0$.
\Eqref{eq:pairwisedynamics} is therefore
$\ddt{\Delta_{ij}}=-2R(t)c_{ij}(t)\Delta_{ij}$.
Integration proves \eqref{eq:pairwisecontract}; taking the maximum over
the finitely many pairs gives the diameter bound.

\end{proof}

\subsection{Proof of Theorem~\ref{thm:condensation}}
\label{app:condensation}
\begin{proof}
Set
\[
 \omega:=\sup_{t\ge T}
              \left|\ddt{\theta_*(t)}\right|<2R_0\sin d.
\]
Lemma~\ref{lem:individual}, applied to every $k\in G$, shows that
all group directions remain in the stated band.
Lemma~\ref{lem:pairwise} then gives \eqref{eq:condensation}.
The positive rate $c_d=2R_0\cos d$ therefore implies that the
diameter tends to zero as $t\to\infty$.

If additionally $\ddt{\theta_*}(t)\to0$, the tracking estimate
\eqref{eq:individualtracking} implies
$|\theta_k(t)-\theta_*(t)|\to0$ for each $k\in G$.
To see this, split its convolution at a fixed late time $S$.
The part over $[T,S]$ decays exponentially, whereas the part over
$[S,t]$ is at most $c_d^{-1}\sup_{s\ge S}|\ddt{\theta_*}(s)|$.
Letting $t\to\infty$ and then $S\to\infty$ proves the claim.
For negative neurons near a positive maximum,
$q_t=R\cos\delta$ gives $\ddt{\theta_k}=2p_k=-2R\sin\delta_k$,
so the same argument applies.
\end{proof}

\clearpage

\section{Experimental details for the main condensation experiment}
\label{app:controls}
\label{app:trajectory}

\subsection{Target and training data}
The four-point experiment in Figures~\ref{fig:overview},
\ref{fig:three_stages}, and \ref{fig:stage1} fits the target
\begin{equation}
 f_\star(x)=0.2\relu(x-1/3)+0.2\relu(-x-1/3),\qquad x\in[-1,1],
 \label{eq:target}
\end{equation}
using noiseless samples
\[
 {\bm{X}}=(-1,-1/3,1/3,1)^\top,\qquad {\bm{y}}=(2/15,0,0,2/15)^\top.
\]
Test MSE is evaluated on 10,000 equally spaced points in $[-1,1]$.

\subsection{Network and optimization}
The network in \eqref{eq:model} has $m=200$ hidden neurons. All parameters are initialized
independently as $a_k,w_k,b_k\sim\mathcal N(0,m^{-2\gamma})$ with
$\gamma=0.4$. Training uses 80,000 full-batch SGD updates,
learning rate $\eta=0.1$, and weight decay $\lambda=10^{-3}$:
\[
 \Theta_{E+1}=(1-\eta\lambda)\Theta_E
                  -\eta\nabla_\Theta\mathcal L(\Theta_E).
\]

Figure~\ref{fig:stage1} shows the evolution of the NTK eigenvalues
throughout training. The large initialization kernel estimate
underlying Theorem~\ref{thm:stage1}, established in
Appendix~\ref{app:initialkernel}, gives
$K_0/m^{1-2\gamma}\to\bar K\succ0$ in probability.
Thus, for $\gamma<1/2$ and sufficiently large width, all initial
NTK eigenvalues are much larger than the fixed weight decay
coefficient $\lambda$.

Following rapid fitting, their early decline is consistent with
Theorem~\ref{thm:earlycompression}, which predicts
$K_t\approx {\mathrm{e}}^{-2\lambda(t-T_{\mathrm I})}K_{T_{\mathrm I}}$
during early amplitude compression under its stated assumptions.
The eigenvalues subsequently stabilize in Stage III and remain
above $\lambda$ throughout the observed trajectory.

\begin{figure}[!htbp]
 \centering\includegraphics[width=0.8\linewidth]{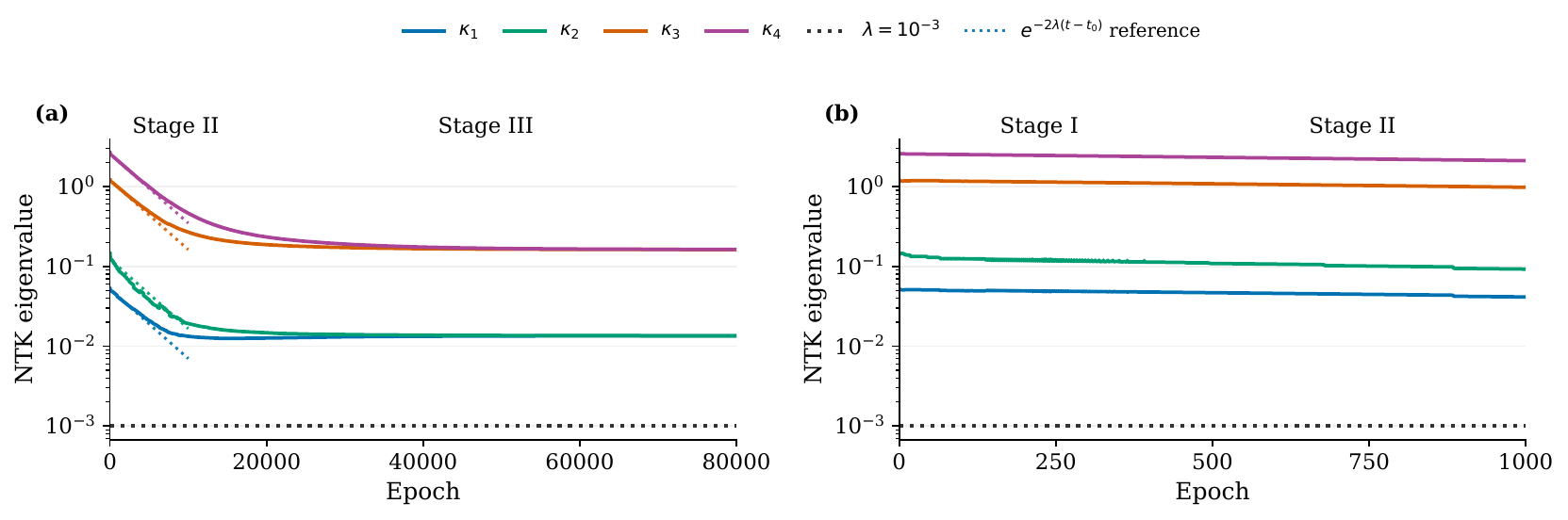}
    \caption{NTK spectrum under large initialization.
    (a) Full training trajectory: all four eigenvalues remain above
    $\lambda=10^{-3}$, decrease during Stage II, and approach
    plateaus during Stage III.
    (b) Zoom into the first 1,000 epochs, showing the comparatively
    small changes in the eigenvalues during initial fitting.
    The horizontal dotted line indicates $\lambda=10^{-3}$.}
 \label{fig:stage1}
\end{figure}

\clearpage
\section{Additional ReLU condensation experiments}
\label{app:numerics}

\subsection{Initialization scale and weight decay}
\label{app:zero-decay}

\begin{figure}[!htbp]
 \centering\includegraphics[width=\linewidth]{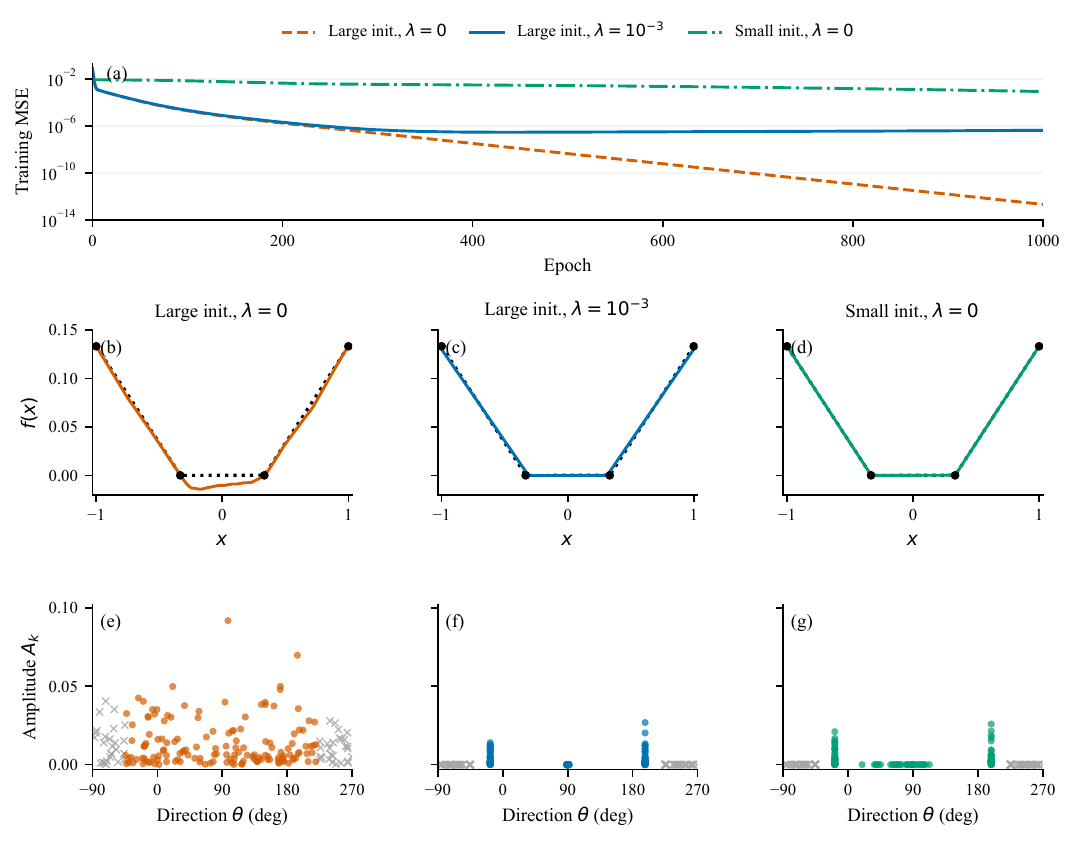}
 \caption{Fitting speed and final condensation at different initialization scales.
 (a) Training MSE during the first 1,000 updates.
 (b--d) Outputs at epoch 80,000, with targets (dotted) and training samples
 (black dots). (e--g) Final directions and unsigned amplitudes $A_k=|a_k|r_k$.
 Columns compare large initialization with $\lambda=0$ and $10^{-3}$,
 and small initialization with $\lambda=0$.
 Colored circles show non-dead neurons; gray crosses show neurons
 inactive on every training input.}
 \label{fig:controls}
\end{figure}

Figure~\ref{fig:controls} compares large initialization
($\gamma=0.4$, $\lambda=0$ or $10^{-3}$) with small initialization
($\gamma=1$, $\lambda=0$) on the four-point target.
Both large initialization with weight decay and small initialization
without decay learn condensed representations with two dominant
angular groups. Their fitting dynamics differ: large initialization
reaches a low training loss within a few hundred updates, well before
late condensation, whereas small initialization fits more slowly.

In the large initialization run with $\lambda=10^{-3}$, all non-dead
neurons become condensed, including a low-amplitude group
near $90^\circ$. Dead neurons remain dispersed while their amplitudes
vanish. With small initialization and no decay, some low-amplitude
non-dead neurons remain dispersed despite the two dominant groups.
Large initialization without decay fits rapidly,but neurons don't develop a condensed representation.

\subsection{Targets with three ReLU breakpoints}
\label{app:multiknot}

\begin{figure}[!htbp]
 \centering\includegraphics[width=\linewidth]{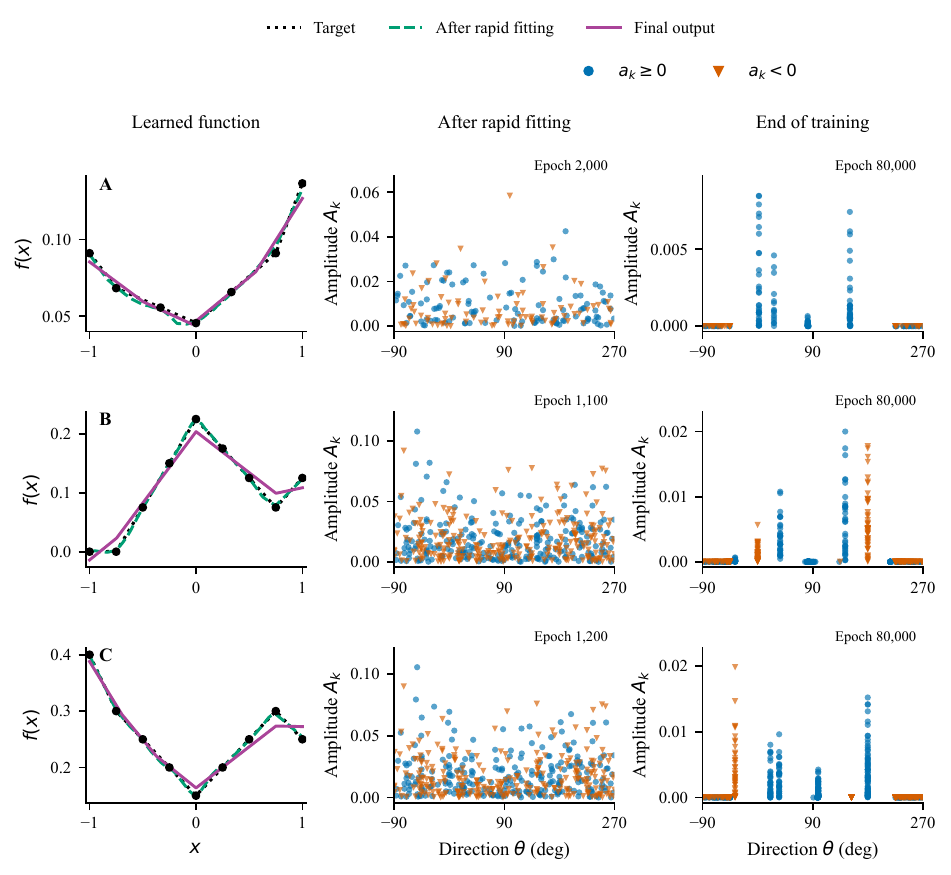}
 \caption{Fitting precedes condensation for the targets in
 \eqref{eq:three-knot-target} (rows A--C).
 Left: target, training samples, early output, and final output.
 Middle/right: neurons after rapid fitting/at epoch 80,000, with
 identical marker sizes and separately scaled amplitude axes.
 Circles and triangles denote nonnegative and negative output weights.}
 \label{fig:multiknot}
\end{figure}

Figure~\ref{fig:multiknot} compares a convex target (A) with targets
whose peak and trough occur in opposite orders (B, C):
\begin{equation}
\begin{aligned}
 f_{\mathrm A}(x)&=0.06\relu(x+3/4)+0.09 \relu(-x)+0.12\relu(x-3/4),\\
 f_{\mathrm B}(x)&=0.3\relu(x+3/4)-0.5\relu(x)+0.4\relu(x-3/4),\\
 f_{\mathrm C}(x)&=0.2\relu(x+3/4)+0.4\relu(-x)-0.4\relu(x-3/4),
\end{aligned}
 \label{eq:three-knot-target}
\end{equation}
where $x\in[-1,1]$.
After rapid fitting, the outputs already track the targets while
neuron directions remain diffuse. Later, neurons with appreciable
amplitudes form narrow angular groups in all three cases.
The final outputs show a regularization bias.

\clearpage
\section{Transformer experiment on anchor function composition}
\label{app:transformer}

This section specifies the experiment in
Figure~\ref{fig:transformer_condensation}. We use the anchor function
framework of \citet{zhang2024initialization}, with the data split and
training settings described below. The comparison consists of two runs
with weight decay coefficients 0 and 0.8, identical saved data, and
identical initial parameters throughout the network.

\subsection{Anchor functions and data construction}
Each input is a sequence of nine integer tokens containing exactly one
consecutive anchor pair $(a,b)\in\{1,2,3,4\}^2$. The token immediately
preceding this pair is the key $x$. The four elementary functions are
\[
 g_1(x)=x+5,\qquad g_2(x)=x+1,\qquad
 g_3(x)=x-2,\qquad g_4(x)=x-8.
\]
The target is the integer $y=g_a(g_b(x))=x+d_a+d_b$, where
$(d_1,d_2,d_3,d_4)=(5,1,-2,-8)$. These translations commute, so their
order does not affect the target. All pairs use this compositional
rule, including $(3,4)$ and $(4,3)$, whose target is $x-10$.

Using zero-based positions, we sample the key position
$p\in\{0,\ldots,6\}$ uniformly and place $(x,a,b)$ at positions
$(p,p+1,p+2)$. The other six input tokens are sampled independently
and uniformly from $\{20,\ldots,99\}$ and do not affect the target.
The key is drawn uniformly from the same integer range, subject to
\[
 x\bmod 9\ne p\quad\text{for training},\qquad
 x\bmod 9=p\quad\text{for testing}.
\]
The modulus is the input length, 9, as implemented in this experiment.
Training contains the 14 anchor pairs other than $(3,4)$ and $(4,3)$;
both excluded pairs are held out entirely. Testing contains all 16
pairs. Thus the test groups for the 14 seen pairs assess generalization
to held-out key--position combinations, while the other two groups
also assess unseen anchor compositions. No non-compositional target
is introduced in this experiment.

\subsection{Network and initialization}

The model is a decoder-only Transformer with two layers and one causal
self-attention head per layer. Given an input sequence represented by
one-hot vectors $X^{\mathrm{in}}$, the input to the Transformer
block is
\begin{equation}
    H^{(0)}
    =
    X^{\mathrm{in}}W^{\mathrm{em}} + X^{\mathrm{pos}},
\end{equation}
where $X^{\mathrm{pos}}$ is a learned positional embedding. For layer $l$, we compute:
\begin{equation}
\begin{aligned}
    \widetilde H^{(l)}
        &= \operatorname{LN}
           \left(H^{(l-1)}\right),\\
    Q^{(l)}
        &= \widetilde H^{(l)}W_Q^{(l)},\\
    K^{(l)}
        &= \widetilde H^{(l)}W_K^{(l)},\\
    V^{(l)}
        &= \widetilde H^{(l)}W_V^{(l)}.
\end{aligned}
\end{equation}
The causal attention matrix and the output of the attention layer are
\begin{equation}
\begin{aligned}
    Attn^{(l)}
    &=
    \operatorname{softmax}\left(
        \frac{Q^{(l)}K^{(l)\top}}{\sqrt{d_k}}
    \right)(\text{with mask}),\\
    U^{(l)}
    &=
    H^{(l-1)}
    +
    Attn^{(l)}V^{(l)}W_O^{(l)},
\end{aligned}
\end{equation}
The feedforward layer is computed as:
\begin{equation}
\begin{aligned}
    \widetilde U^{(l)} &= \operatorname{LN}^{(l)} \left( U^{(l)} \right) ,\\
    H^{(l)}
    &= \widetilde U^{(l)}+\operatorname{ReLU}\!\left(
        \widetilde U^{(l)}W_1^{(l)}
    \right)W_2^{(l)}.
\end{aligned}
\end{equation}
The final logits are
\begin{equation}
    Z = H^{(N)}W^{\mathrm{out}}+b^{\mathrm{out}}.
\end{equation}
LN refers to Layer Normalization.

Each linear weight matrix
$W\in\mathbb{R}^{d_{\mathrm{in}}\times d_{\mathrm{out}}}$ is initialized
entrywise as
\begin{equation}
    W_{ij}
    \sim
    \mathcal N\!\left(0,d_{\mathrm{in}}^{-2\gamma}\right)
\end{equation}

\begin{table}[htbp]
 \centering\small
 \caption{Settings shared by the two Transformer runs in
 Figure~\ref{fig:transformer_condensation}; only weight decay differs.}
 \label{tab:transformer-settings}
 \begin{tabular}{ll}
 \toprule
 Setting & Value\\
 \midrule
 Layers / heads per layer & 2 / 1\\
 Embedding / feedforward dimension & 400 / 1,200\\
 Query, key, and value dimension & 64 \\
 Input length / vocabulary size & 9 / 120\\
 Initialization rate & 0.48 \\
 Optimizer & AdamW \\
 Batch size / training epochs & 4,096 / 400\\
 Learning rate: initial / peak / minimum & $5\times10^{-5}$ / $5\times10^{-4}$ / $5\times10^{-5}$\\
 Warmup / cosine schedule & 40 / 360 epochs\\
 Global gradient-norm clipping & 1.0\\
 \bottomrule
 \end{tabular}
\end{table}

\subsection{Training objective and weight decay comparison}
Training minimizes cross entropy between the logits at the last output position and the appended
target $y$. AdamW applies decoupled weight decay to \emph{all} trainable parameters,
including embeddings, LayerNorm parameters, and the output bias;
there are no parameter groups exempted from decay. The learning rate
warms up from $5\times10^{-5}$ to $5\times10^{-4}$ over 40 epochs,
then follows a cosine schedule with period parameter 360 epochs and
minimum $5\times10^{-5}$. 

\subsection{Cosine similarities}

\begin{figure}[htbp]
    \centering
    \includegraphics[width=\textwidth]{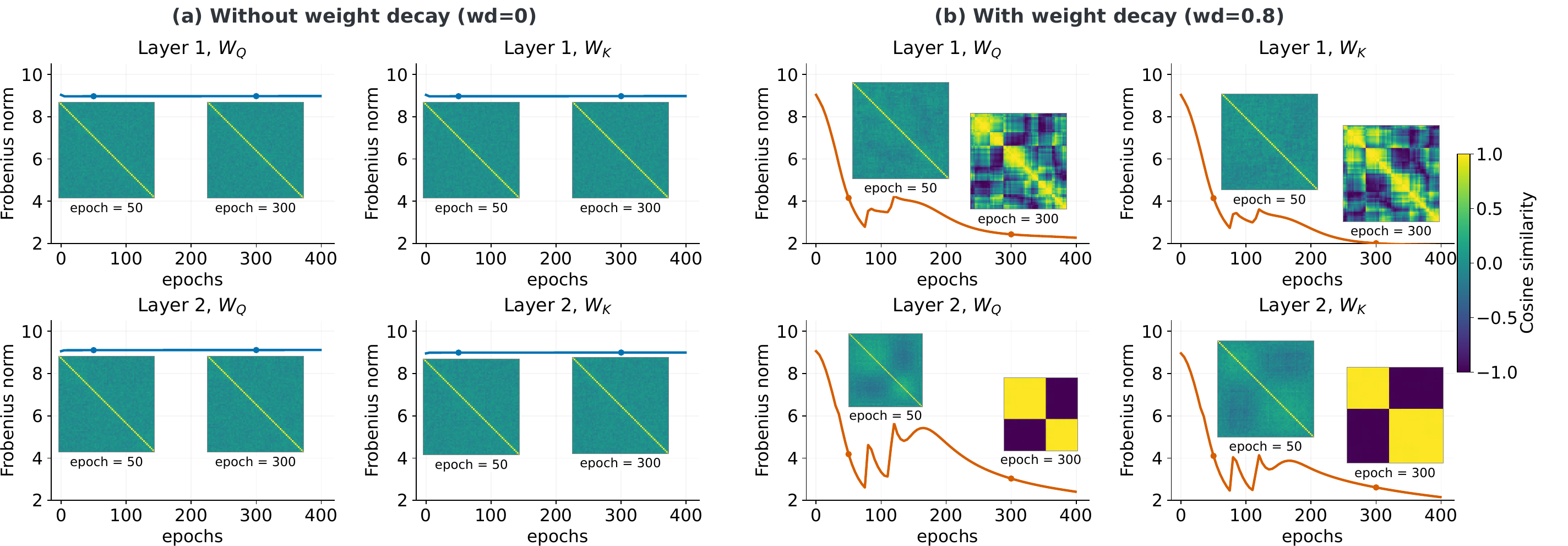}
    \caption{Weight contraction and condensation across Transformer layers.
    The top and bottom rows show layers 1 and 2, respectively.
    Panel (a) corresponds to training without weight decay
    ($\lambda=0$), and Panel (b) corresponds to training with
    weight decay ($\lambda=0.8$). Within each panel, the columns show
    $W_Q$ and $W_K$. Curves show Frobenius norms during training, while
    the insets show row-wise cosine similarities at epochs 50 and 300.
    Weight decay induces norm contraction and structured directional
    alignment in both layers.}
    \label{fig:transformer-condensation-layers}
\end{figure}

In our experiment, $W_Q,W_K\in\mathbb R^{64\times400}$. For rows ${\bm{w}}_i,{\bm{w}}_j\in\mathbb R^{400}$, define cosine similarity
\[
 C_{ij}=\frac{{\bm{w}}_i^\top {\bm{w}}_j}{\|{\bm{w}}_i\|_2\,\|{\bm{w}}_j\|_2}.
\]
Rows and columns are reordered by the same permutation and grouped according to cosine similarity.
Figure~\ref{fig:transformer_condensation} uses the query and key projection matrices from the first Transformer blocks. 
Figure~\ref{fig:transformer-condensation-layers} shows result on both Transformer layers.

\end{document}